\documentclass[journal]{IEEEtran} 
\usepackage{url}
\usepackage{amsmath, amsfonts, amssymb, amsthm}
\usepackage{algorithm2e}
\usepackage{enumerate}
\usepackage{caption}
\usepackage{subcaption}
\usepackage{graphicx}
\usepackage{longtable,tabularx}
\usepackage{placeins} 
\usepackage{float}
\usepackage{multirow}
\usepackage{bbm}
\usepackage[table]{xcolor}
\usepackage{threeparttable}
\usepackage{balance}
\usepackage[inkscapelatex=true]{svg}

\newtheorem{theorem}{Theorem}
\newtheorem{assumption}{Assumption}
\newtheorem{proposition}{Proposition}
\newtheorem{definition}{Definition}
\newtheorem{corollary}{Corollary}
\newtheorem{remark}{Remark}

\newcommand{\R}{\mathbb{R}}
\newcommand{\mb}[1]{\mathbf{#1}}

\title{\textbf{CATNIPS}: \textbf{C}\textmd{ollision} \textbf{A}\textmd{voidance}  \textbf{T}\textmd{hrough} \textbf{N}\textmd{eural} \textbf{I}\textmd{mplicit}
\textbf{P}\textmd{robabilistic}
\textbf{S}\textmd{cenes}}

\author{Timothy Chen$^1$, Preston Culbertson$^2$, Mac Schwager$^1$
\thanks{*The NASA University Leadership Initiative (grant \#80NSSC20M0163) provided funds to assist the authors with their research, but this article solely reflects the opinions and conclusions of its authors and not any NASA entity. Toyota Research Institute provided funds to support this work.  The first author was supported by a NASA NSTGRO Fellowship, and the second author was supported on a NASA NSTRF Fellowship. }% <-this % stops a space
\thanks{$^{1}$Department of Aeronautics and Astronautics, Stanford University, Stanford, CA 94305, USA
        {\tt\small \{chengine, schwager\}@stanford.edu}}%
\thanks{$^{2}$Department of Mechanical and Civil Engineering, California Institute of Technology, Pasadena, CA 91125, USA, {\tt\small pculbert@caltech.edu}}%
        }

\begin{document}
\maketitle

%===============================================================================
\begin{abstract}
    We introduce a transformation of a Neural Radiance Field (NeRF) to an equivalent Poisson Point Process (PPP). This PPP transformation allows for rigorous quantification of uncertainty in NeRFs, in particular, for computing collision probabilities for a robot navigating through a NeRF environment. The PPP is a generalization of a probabilistic occupancy grid to the continuous volume and is fundamental to the volumetric ray-tracing model underlying radiance fields.  Building upon this PPP representation, we present a chance-constrained trajectory optimization method for safe robot navigation in NeRFs. Our method relies on a voxel representation called the  Probabilistic Unsafe Robot Region (PURR) that spatially fuses the chance constraint with the NeRF model to facilitate fast trajectory optimization.  We then combine a graph-based search with a spline-based trajectory optimization to yield robot trajectories through the NeRF that are guaranteed to satisfy a user-specific collision probability.  We validate our chance constrained planning method through simulations and hardware experiments, showing superior performance compared to prior works on trajectory planning in NeRF environments. Our codebase can be found at \url{https://github.com/chengine/catnips}, and videos can be found on our project page (\url{https://chengine.github.io/catnips}). 
\end{abstract}

% KEYWORDS

%===============================================================================

\section{Introduction}

\begin{figure}[]
         \centering         \includegraphics[width=0.49\textwidth]{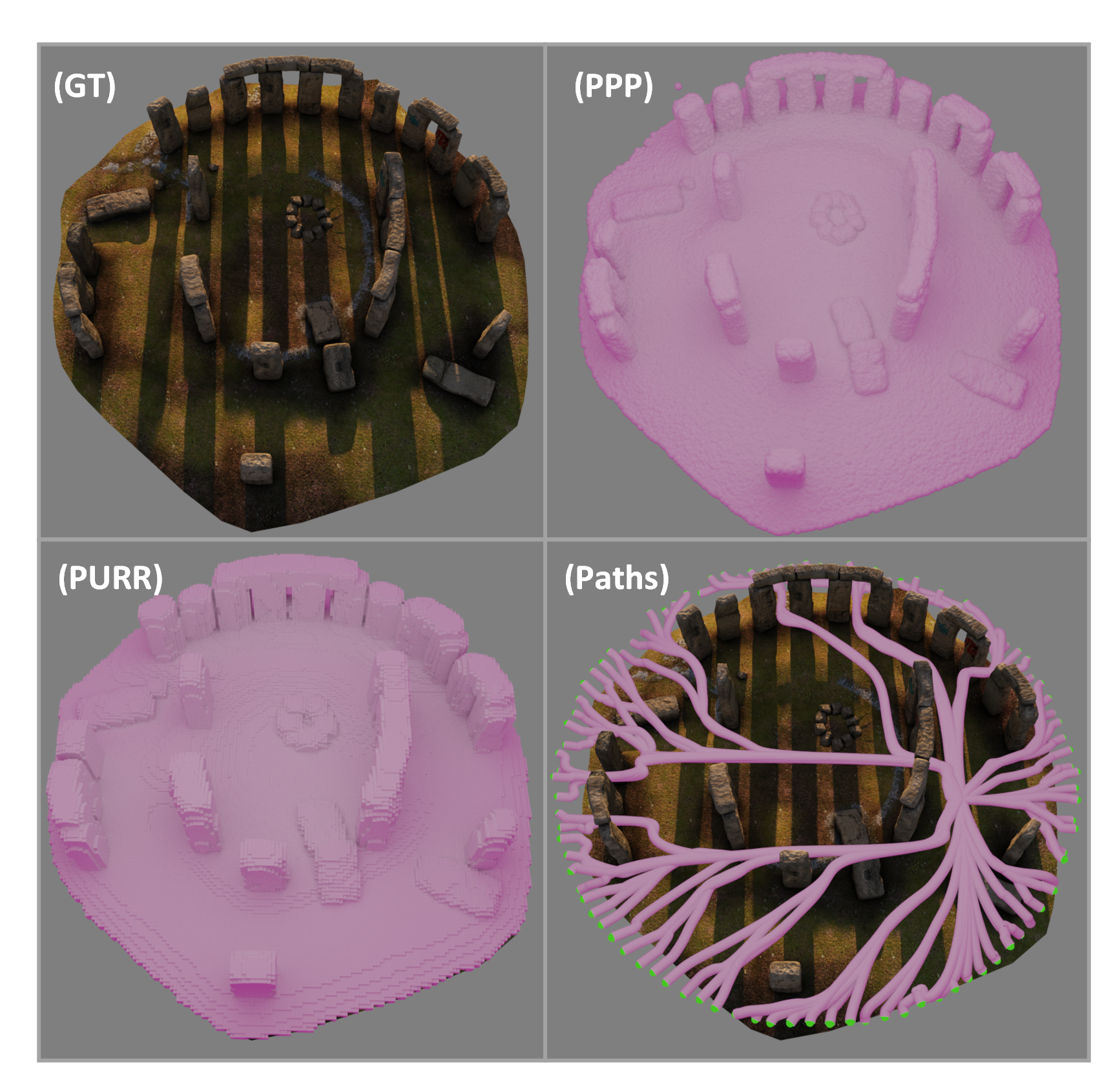}
        \caption{(a) Ground-truth of the Stonehenge scene, (b) Poisson Point Process (PPP) of the scene represented as a point cloud, (c) Probabilistically Unsafe Robot Region (PURR) of scene, (d) Generated safe paths from CATNIPS.}
        \label{fig:three graphs}
\end{figure}

Constructing an environment model from onboard sensors, such as RGB(-D) cameras, lidar, or touch sensors, is a fundamental challenge for any autonomous system. Recently, Neural Radiance Fields (NeRFs) \cite{mildenhall2020a} have emerged as a promising 3D scene representation with potential applications in a variety of robotics domains including SLAM \cite{sucarIMAPImplicitMapping2021}, pose estimation \cite{yen2021inerf, yen2022nerfsupervision}, reinforcement learning \cite{driess2022learning}, and grasping \cite{ichnowski2021dex}. NeRFs offer several potential benefits over traditional scene representations: they can be trained using only monocular RGB images, they provide a continuous representation of obstacle geometry, and they are memory-efficient, especially considering the photo-realistic quality of their renders. Using current implementations \cite{mueller2022instant, nerfstudio}, NeRFs can be trained in seconds using only RGB images captured from monocular cameras, making onboard, online NeRF training a viable option for  robotic systems.

However, NeRFs do not directly give information about spatial occupancy, which poses a challenge in using NeRF models for safe robot navigation. In other 3D scene representations, such as (watertight) triangle meshes \cite{edelsbrunner2003surface}, occupancy grids \cite{elfes1989using}, or Signed Distance Fields (SDFs) \cite{osher2003level}, occupancy is well-defined and simple to query. NeRFs, however, do not admit simple point-wise occupancy queries, since they represent the scene geometry implicitly through a continuous volumetric density field. For this reason, integrating NeRF models into robotic planners with mathematical safety guarantees remains an open problem.

To this end, we develop a framework for robot trajectory planning that can generate trajectories through a NeRF scene with probabilistic safety guarantees. To do this, we propose a mathematical transformation of a NeRF to a Poisson Point Process (PPP), which allows for the rigorous computation of collision probabilities for a robot moving through a NeRF scene. We further introduce a novel scene representation, a Probabilistically Unsafe Robot Region (PURR), that convolves the robot geometry with the NeRF to yield a 3D map of all robot positions with collision probabilities less than a user-specified threshold. Finally, we propose a fast, chance-constrained trajectory planner that uses the PURR to ensure the trajectories are collision free up to the user-specified probability threshold. Our method, called CATNIPS, can compute probabilistically safe trajectories at more than 3 Hz. This is many times faster than existing NeRF-based trajectory planners that provide no safety guarantees \cite{adamkiewicz2022vision}.  

The key theoretical advance underpinning our results is the novel transformation of the NeRF into a PPP.  Existing works on radiance fields either ignore the underlying probabilistic interpretation of the field or treat it as a nuisance.  %Some works constrain the expected distance between robot and the NeRF scene \cite{tong2022enforcing} or represent the environment geometry using a triangle mesh approximating a given level set of the NeRF density \cite{sucarIMAPImplicitMapping2021}.  
A naive approach is to convert the NeRF representation into a more traditional deterministic mesh or occupancy representation.  We argue that such conversions are computationally slow, and they destroy any potential mathematical safety guarantees for a downstream planner.  For example, generating a triangle mesh (e.g., using marching cubes \cite{lorensen1987marchinga}) that represents a level set of the density field requires the arbitrary selection of a density cutoff value, and collapses the uncertainty represented by the density field into a binary occupancy measure.  In contrast, our method computes rigorous collision probabilities using the NeRF density directly. %In particular, by analyzing the volumetric rendering process used to generate images from the NeRF, we show the density field can be interpreted as a Poisson Point Process (PPP), a stochastic process on $\mathbb{R}^3$. The PPP allows us to quantify the uncertainty distribution of obstacle geometry in the environment and to directly compute the probability that any portion of a robot body is in collision with the environment. We propose a new map representation (the PURR) obtained directly from the NeRF based on these collision probability computations, and give an accompanying planning method for planning trajectories for a robot through the scene that are guaranteed to satisfy a user-specified collision probability constraint. 

We provide simulation studies to show that our planner generates safe, but not overly-conservative, trajectories through the environment. We contrast our paths to those generated using a level-set based environment representation and those from prior work \cite{adamkiewicz2022vision}. We find that the paths our method generates are more intuitive and easier to tune than these baselines, as collision is interpretably defined through violation of a collision probability as opposed to violation of an arbitrary level set of the density. We show our method to be real-time, replanning online at 3 Hz on a laptop computer, compared to the gradient descent-based planner proposed in \cite{adamkiewicz2022vision}, which requires approximately 2 seconds for replanning. 

The rest of this paper is organized as follows.  In Section~\ref{Sec:RelatedWork} we discuss related work.  In Section~\ref{Sec:Background} we review background concepts from  NeRFs, and in Section~\ref{Sec:PPP} we derive the Poisson Point Process interpretation of the NeRF.  In Section~\ref{Sec:CollisionProbability} we compute collision probabilities for a robot in a NeRF environment, and in Section~\ref{Sec:CollisionAvoidance} we present our trajectory planning algorithm, CATNIPS.  Section~\ref{Sec:Results} gives our simulation results and Conclusions are in Section~\ref{Sec:Conclusion}.

\section{Related Work}
\label{Sec:RelatedWork}
Here we review the related literature in robot planning and control with NeRF representations, compare it with planning in a Signed Distance Function (SDF) representation, discuss other uses of NeRFs in robotics, and summarize chance-constrained planning.

\subsection{Planning and Control with Onboard Sensing and SDFs}
Planning and control based on onboard sensing has already yielded a large amount of literature. Typically these works present reactive control schemes \cite{perceptionaware}, using the sensed depth directly to perform collision checking in real-time. These methods typically are myopic, reasoning only locally about the scene. An alternate approach is to construct a map of the environment using the depth measurements. Often a Signed Distance Field (SDF) is constructed from depth data \cite{voxblox}, \cite{fiesta}, which in this work is encoded within voxels. \cite{fiesta} also integrates their system onboard a quadrotor to validate their method. Such a representation is typical in dynamic robotic motion planning, providing fast collision checking and gradients in planning. 

We believe NeRF is a promising alternative to more familiar 3D geometry representations like SDFs due to some key NeRF properties. We show that the NeRF inherently encodes uncertainty in the environment, whereas  SDFs are typically deterministic. Moreover, we find that deep network SDFs are difficult to train, often requiring synthetic training points with heuristically generated, error-prone depth labels. In contrast, NeRFs can be supervised directly from RGB images, and can be trained reliably and quickly with NeRF training packages such as \cite{nerfstudio}. The modularity of NeRFs in perception pipelines, especially those involving visual data, is another benefit of NeRFs. However, this is not to say that the two cannot coexist. There exists in the literature deep learning architectures that simultaneously learn SDF and NeRF outputs based on empirical consistency between the two (e.g., NeuS \cite{wang2021neus}). We hope that our probabilistic interpretation of NeRFs can help bridge the gap between these two representations and enable future pipelines to access advantages of both representations.

\subsection{Planning and Control using NeRFs}

Safety has been a largely unexplored topic in the NeRF literature, with only preliminary approaches being studied in simulation. The authors' previous work NeRF-Nav \cite{adamkiewicz2022vision} presents a planner that avoids collisions in a NeRF environment model by avoiding high-density areas in the scene. An alternative work \cite{tong2023enforcing} instead uses the predicted depth map at sampled poses to enforce step-wise safety using a control barrier function. The two methods are not at odds, as the philosophy of \cite{adamkiewicz2022vision} serves as a high-level planner that encourages non-myoptic behavior while \cite{tong2023enforcing} can be used as a safety filter for a myopic low-level controller interfacing directly with the system dynamics.

More specifically, NeRF-Nav \cite{adamkiewicz2022vision} adapts trajectory optimization tools to plan trajectories for a robot through a NeRF environment. Collisions are discouraged with a penalty in the trajectory cost, but the probability of collision is not quantified or directly constrained. In this work, we instead rigorously quantify collision probabilities for a robot in a NeRF, and develop a trajectory planning method to satisfy user-defined chance constraints on collision. In addition, NeRF-Nav requires about 2 seconds for each online trajectory re-solve, while our proposed method requires about $0.3$ seconds per online trajectory re-solve on similar computing hardware.  %The previous work considered collisions as a penalty with a weight (which may not avoid collisions at all depending on the size of the weight with respect to the rest of the cost), but since collision is catastrophic in most of the considered applications it is desirable to be able to impose this as a hard constraint. Moreover, the proposed method is real-time, feasible, and does not suffer from local optima, in contrast to the slow gradient descent method of \cite{adamkiewicz2022vision} and the potential infeasibility of solving for plans in \cite{tong2022enforcing}.

\subsection{Other Uses of NeRFs in Robotics}
Some works have considered NeRFs as a 3D scene representation for robotic grasping and manipulation. For example, Dex-NeRF \cite{ichnowski2021dex} uses NeRF-rendered depth images to obtain higher-quality grasps for a robot manipulator than using a depth camera. Similarly, one can use dense object descriptors supervised with a NeRF model  for robot grasping \cite{yen2022nerfsupervision}. %Recently, grasp metrics accounting for the uncertainty in NeRFs (CITE Preston's paper if available) have also been proposed.  

Some works have also considered SLAM and mapping using a NeRF map representation. The papers  \cite{sucarIMAPImplicitMapping2021}, \cite{yen2021inerf} use the photometric error between rendered and observed images to simultaneously optimize the NeRF weights and the robot/camera poses. The approach in \cite{nice-slam} uses a grid-interpolation-decoder NeRF architecture in a similar SLAM pipeline. The work \cite{nerf-slam} proposes a combination of an existing visual odometry pipeline for camera trajectory estimation together with online NeRF training for the 3D scene.  NeRFs have also been used for tracking the pose of a robot using an on-board camera and IMU.  For example, iNeRF \cite{yen2021inerf} finds a single camera pose from a single image and a pre-trained NeRF model, and \cite{adamkiewicz2022vision} proposes a nonlinear optimization-based filter for tracking a trajectory of an on-board camera using a sequence of images and a pre-trained NeRF. Loc-NeRF \cite{loc-nerf} approaches a similar problem using a particle filter instead of a nonlinear optimization-based filter.

Other papers have considered active view planning for NeRFs. ActiveNeRF \cite{active-nerf} treats the radiance value of the NeRF as Gaussian distributed random variables, and performs Bayesian filtering to find the next best view. An alternative work \cite{nerfensembles} uses disagreement among an ensemble of NeRFs to choose the next best camera view. Similarly, \cite{densitynerf} uses ensembles in a next best view strategy while also adding ray-termination densities to the information gain metric. S-NeRF \cite{snerf} uses variational inference to train a probability distribution over NeRFs for next best view selection. ActiveRMAP \cite{activermap} considers a full informative trajectory planning pipeline for a robot moving through a NeRF. In contrast to our work here, they do not focus on the safety of the trajectories or on quantifying collision probabilities. Perhaps the closest in spirit in terms of modelling uncertainty is Bayes' Rays \cite{goli2023bayes}, which reasons locally about epistemic uncertainty (i.e. the distribution over NeRF parameters) and differs from our work that pins down the distribution that the NeRF model represents. However, we envision future efforts that incorporate both works to fully explain geometric uncertainty conditioned on RGB data.

Of course, many of these works would not be applicable to robotics if they were not real-time. Massive performance gains have been made to train NeRFs in real-time \cite{mueller2022instant}, \cite{plenoxels}, \cite{plenoctrees}. Moreover, NeRFs must be able to capture reality as well. They are known to suffer in quality when reconstructing rich, real environments over a large range of length scales. Attempts have been made to fix this issue by extrapolating over the entire camera frustum rather than a ray \cite{mipnerf}, \cite{mipnerf360}. 

\subsection{Chance-constrained Planning}
Outside of NeRFs, there exists a large literature on trajectory planning for robots that seeks to impose constraints on the probability of collision when the underlying scene geometry is unknown; this approach is known as chance-constrained planning. The robot state is typically modeled as stochastic, while the map is typically considered to be known deterministically. Some works do consider uncertainty in both the map and the robot state, but they typically rely on a linear system or Gaussian noise assumption to make computation convex or analytical and efficient to solve. Du Toit and Burdick \cite{toitCC} assume Gaussian-distributed obstacle states, and approximates the collision probability as constant over the robot body (suitable only for small robots). Blackmore et al. \cite{blackmoreCC} encode the probability of collision with faces of polytopic obstacles as a linear constraint, but the resulting trajectory optimization is a combinatorial problem, making it difficult to solve quickly. Zhu and Alonso-Mora \cite{zhuCC} incorporate this linear probabilistic constraint into RRT. Luders et al. \cite{ludersCC} again use this linear constraint in an MPC framework with nonlinear dynamics, executed on real hardware with dynamic obstacles and extended to a multi-agent context.  None of these methods consider NeRF environment models, which is our focus here.

\section{Neural Radiance Fields (NeRFs)}
\label{Sec:Background}
In this section, we introduce the mathematical preliminaries and notation used in NeRFs. For clarity, we use bold face for vector variables and functions that output vectors, and non-bold text for scalar variables, functions that output scalars, and---in some instances---rotations.

A NeRF is a neural network that stores a density and color field over the 3D environment.  When coupled with a differentiable image rendering model (usually a differentiable version of ray tracing), the NeRF can be trained from a collection of RGB images with known camera poses, and can generate photo-realistic synthetic images rendered from camera view points that are different from the training images. 

More specifically, the NeRF is a pair of functions $(\rho(\mathbf{p}),c(\mathbf{p},\mathbf{d}))$.  The density function, $\rho:\mathbb{R}^3 \mapsto \mathbb{R}_{\ge 0}$, maps a 3D location $\mathbf{p} = (x, y, z)$ to a non-negative density value $\rho$ that encodes the differential probability of a light ray stopping at that point.\footnote{This density field can be stored entirely as a multi-layer perceptron (MLP), as in the original NeRF work \cite{mildenhall2020a}, as a function interpolated on a discrete voxel grid \cite{plenoxels}, or using a combination of interpolated voxel features and an MLP decoder \cite{mueller2022instant,plenoctrees}.  Our method can work with any of these representations.}  The radiance (i.e., RGB color) function $\mb{c}:\mathbb{R}^3 \times \mathbb{R}^2 \mapsto \mathbb{R}^3$ maps a 3D location $\mathbf{p} = (x, y, z)$ and camera view direction $\mathbf{d} \in \{\bold{x} \in \mathbb{R}^3 \;|\; ||\bold{x}|| = 1\}$ (alternatively parameterized as a 2D vector of angles $(\theta, \phi)$) to an emitted RGB color $\mb{c}$ represented as a vector in $\mathbb{R}^3$. In this paper, we focus specifically on the density function $\rho(\mathbf{p})$ as a proxy for occupancy, which should ideally be zero in free space and take on large values in occupied space. 
 We use this $\rho(\mathbf{p})$ function as a map representation for planning robot trajectories.  We also define $\mb{C}(\mathbf{o}, \mathbf{d}) \in [0, 1]^3$ as the rendered pixel color in an image when taking the expected color value from the NeRF along a ray $\mathbf{r}(t; \mathbf{o}, \mathbf{d})$ with camera origin $\mathbf{o}$ and pixel orientation $\mathbf{d}$, where $\mathbf{r}(t) = \mathbf{o} + t\cdot \mathbf{d}$. The rendered color is given by

\begin{equation}
\label{Eq:RayTracing}
\mb{C}(\mathbf{o}, \mathbf{d})= \int_{t_n}^{t_f} \rho\left(\mathbf{r}(t)\right) \; e^{-\int_{t_n}^{t} \rho(\mathbf{r}(\tau)) d\tau} \; \mathbf{c}(\mathbf{r}(t), \mathbf{d})\; dt,
\end{equation} 
where we only integrate points along the ray between $t_n$ and $t_f$ (i.e. the near and far planes).  In practice, this integral is evaluated numerically using Monte Carlo integration with stratified sampling.  The resulting rendering equation (\ref{Eq:RayTracing}) is differentiable.

A rendered image $I_i$ is then an array of pixel colors associated with a single camera pose, where the color of pixel $j$ in image $I_i$ is given by $\mb{C}(\mathbf{o}_i, \mathbf{d}_{ij})$ with associated origin $\mathbf{o}_i$ (determined by the camera) and direction $\mathbf{d}_{ij}$ computed with an angular offset from the camera optical axis for pixel $j$. We denote the set of pixel indices for image $I_i$ as $\mathcal{I}_i$.  The corresponding ground truth image $\bar{I}_i$ is an array of pixels with colors $\mb{\bar{C}_{ij}}$.  A dataset D for training a NeRF consists of a collection of such ground truth images with known poses. The parameters of the NeRF are trained by minimizing the loss function
\begin{equation}
\label{Eq:PhotometricLoss}
J(\boldsymbol{\theta}) = \frac{1}{|D|}\sum_{i\in D}\frac{1}{|\mathcal{I}_i|}\sum_{j\in \mathcal{I}_i}||\mb{C}(\mathbf{o}_i, \mathbf{d}_{ij}; \boldsymbol{\theta}) - \mb{\bar{C}_{ij}}||_2^2,
\end{equation}
where $\boldsymbol{\theta}$ are the parameters of the neural networks representing the density and radiance fields $\rho$ and $\mb{c}$, which appear in the computation of the pixel color $\mb{C}(\mathbf{o}_i, \mathbf{d}_{ij}; \boldsymbol{\theta})$ through the rendering equation (\ref{Eq:RayTracing}). This mean squared error is called the photometric error (or photomoetric loss) and is optimized with standard stochastic gradient descent tools in, e.g., Pytorch. Intuitively, the goal is to train the network so that the synthetic images generated from the NeRF match the training images at the specified camera poses as closely as possible.

While the camera poses are required to find $\mathbf{o}_i$ and $\mathbf{d}_{ij}$ for each pixel to train the NeRF, a standard pipeline has emerged that takes images without camera poses, uses a classical structure-from-motion algorithm (e.g., COLMAP \cite{schonberger2016structure}) to estimate the camera poses, and supervises the NeRF training with these poses. Recent methods also optimize the camera poses jointly with the NeRF weights to improve performance \cite{nerfstudio}. 

Hence, in practice a NeRF model can be obtained from only RGB images (without camera poses).  However the quality and extent of the NeRF is limited by the quality of the training images, and the volume covered by those images.  Few images, with low resolution, low photographic quality, and poor coverage will yield a poor-quality NeRF.  A large number of sharp, high-resolution images from a rich diversity of view points will yield a high-quality NeRF.  Our goal is to accurately quantify collision risk for a robot navigating through the scene regardless of the quality of the trained NeRF.  With our approach, the same robot pose in the same 3D scene may have a high collision probability in a poor-quality NeRF, and a low collision probability in a high-quality NeRF.  The probability of collision is itself an expression of the NeRF quality in the vicinity of the robot.

% This compact encoding has been used for robotics applications such as pose estimation \cite{yen2021inerf}, SLAM \cite{sucarIMAPImplicitMapping2021}, navigation \cite{adamkiewicz2022vision}, grasping \cite{ichnowski2021dex} among many others. Applications of NeRFs, and more broadly neural rendering, extend beyond robotics. New technologies such as self-driving cars, augmented reality, and unmanned drone delivery all require algorithms to be able to spatially reason about objects in the environmental, and NeRFs have several advantages over traditional representations such as compact memory, photorealistic rendering, and differentiability.

\section{NeRF Density as a Poission Point Process}
\label{Sec:PPP}
In this section, we show that a NeRF density field can be transformed into the density of a Poisson Point Process (PPP), and the NeRF color and density fields together give rise to a ``marked'' PPP \cite{Kingman:1993, chiu2013}. To do this, we demonstrate that the NeRF volumetric rendering equation is precisely the computation that is required to compute expected pixel color if the color and density under this marked PPP model. Training the NeRF can be interpreted as fitting the PPP density parameters through moment matching on the expected pixel color.

This connection is significant since the PPP derived from the NeRF density field enables computation of probabilistic quantities, such as the probability of a given volume being occupied (e.g., of a robot body colliding with the NeRF), or the entropy in the NeRF model. This also settles a debate in the literature about whether the NeRF density can be probabilistically (it can), and paves the way for practical utility in other domains beyond safety (e.g., in active sensing and active view planning). In short, we find that the NeRF density encodes a probabilistic model of the geometry of the scene, the uncertainties of which can be rigorously quantified through an underlying PPP.

\subsection{Poisson Point Processes}

\label{subsec:ppp_background}

Here we review the definition and properties of the Poisson Point Process (PPP), a stochastic process that models the distribution of a random collection of points in a continuous space. Much of this discussion is drawn from \cite{Kingman:1993}, to which we refer the reader for a more detailed and rigorous treatment.

First, we recall that a discrete random variable (RV) $N$ that takes values in $\mathbb{N}$ is said to have a Poisson distribution with parameter $\lambda \geq 0$ if its probability mass function is given by
\begin{equation*}
Pr(N = m) = \frac{\lambda^m \exp(-\lambda)}{m!}.
\end{equation*}
Poisson RVs are often used to model the distribution of the number of discrete events in a fixed amount of time (e.g., customers arriving at a store), or over a fixed region of space (e.g. the number of rides hailed daily in a given neighborhood). The PPP naturally extends this concept to the distribution over the number of points in any subset of a multi-dimensional Euclidean space.  
%\begin{definition}[Homogeneous Poisson Point Process]
%Consider a random process $N$ on $\R^n$ that maps Lebesgue measureable subsets $B \subset \R^n$ of the state space to the (random) number $N(B)$ of points that lie in $B$. We say $N$ is a Poisson Point Process (PPP) with intensity $\lambda > 0$ if:
% \begin{enumerate}[(i)]
% \item The number of points $N(B)$ that lie in $B$ is a Poisson RV with distribution
% \begin{equation*}
% Pr(N(B) = m) = \frac{\mu^m \exp(-\mu)}{m!},
% \end{equation*}
% where $\mu = \lambda \lvert B \rvert$, and $\lvert B \rvert$ denotes the Lebesgue measure of $B$.
% \item For $k$ disjoint subsets $B_1, \ldots B_k \subset \R^n$, the number of points in each subset, $N(B_1), \ldots, N(B_k)$, are independent RVs. 
% \end{enumerate}
% \end{definition}
\begin{definition}[Poisson Point Process]
\label{def:ippp}
Consider a random process $N$ on $\R^n$ that maps subsets\footnote{The subsets $B$ must be Lebesgue measureable.} $B \subset \R^n$ of the state space to the random number $N(B)$ of points that lie in $B$. We say $N$ is a Poisson Point Process (PPP) with intensity $\lambda: \R^n \mapsto \R_{+}$ if:
\begin{enumerate}[(i)]
\item The number of points $N(B)$ that lie in $B$ is a Poisson RV with distribution
\begin{equation*}
Pr(N(B) = m) = \frac{\Lambda(B)^m \exp(-\Lambda(B))}{m!},
\end{equation*}
where $\Lambda(B) = \int_{\mb{x}\in B} \lambda(\mb{x}) \; d\mb{x}$. 
\item For $k$ disjoint subsets $B_1, \ldots B_k \subset \R^n$, the number of points in each subset, $N(B_1), \ldots, N(B_k)$, are independent RVs. 
\end{enumerate}
\end{definition}
This is sometimes refered to as the \emph{inhomogeneous} PPP since the intensity $\lambda$ is a function of the spatial variable $\mb{x}$.  If the intensity is constant over $\mb{x}$, this is called the \emph{homogeneous} PPP. 

The PPP encodes the randomness over both the \emph{number} and the \emph{location} of random points. An important quantity for such processes is the ``void probability,'' or the probability that a given set $B$ is empty. The void probability is given by
\begin{equation}
Pr(N(B) = 0) = \exp\left[-\int_{B}\lambda(\mb{x}) \; d\mb{x}\right]. \label{eq:void}
\end{equation}
Thus, intuitively, the void probability shrinks as either the intensity $\lambda$ increases, or the set $B$ grows larger.

It is also important to note that through the Poisson distribution,  the expected number of points in the set $B$ is identical to the integral of the intensity, in other words,
\begin{equation}
\mathbb{E}\left[ N(B)\right] = \int_{B}\lambda(\mb{x}) \; d\mb{x}.
\label{eq:expected_points}
\end{equation}

When reasoning about collision probability, we want each point associated with the PPP to have a corresponding volume. Using the above fact \eqref{eq:expected_points}, we can consider points or particles of arbitrary size by weighting the PPP accordingly.

Given some reference particle size $V_{ref}$ associated with the initial PPP, a desired particle size $V_d$, an integration domain $B$, and the expected volume occupied by all the particles $\mathbb{E}\left[V_{total}\right]$, we can retrieve $\mathbb{E} \left[N_d(B) \right]$, the expected number of particles of size $V_d$, by conservation of $\mathbb{E}\left[V_{total}\right]$. Namely, 
\begin{align}
% \begin{split}
V_{ref} \mathbb{E} \left[N(B) \right] &= V_{ref} \int_{B}\lambda(\mb{x}) \;d\mb{x} = \mathbb{E}\left[V_{total}\right]\nonumber \\
&= V_d \mathbb{E} \left[N_d(B) \right]\nonumber,\\
\mathbb{E} \left[N_d(B) \right] &= \frac{\int_{B}\lambda(\mb{x})\;d\mb{x}}{V_d / V_{ref}}.
\label{eq:PPP_weighting}
% \end{split}
\end{align}
Therefore, a PPP with the same expected occupied volume can be produced using differing particle sizes by simply scaling the density.

Finally, we note that PPPs may be ``marked'' or ``colored'' with various quantities using a deterministic labeling function $\mb{c}(\mb{x}): \R^n \mapsto \mathcal{C}$. Using the statistics of the underlying PPP, it is straightforward to compute the statistics of the labels for the points appearing in a set; see \cite[Ch.~5]{Kingman:1993} for a detailed discussion.

\subsection{``Rendering'' a Marked PPP}\label{ppp-proof}

\begin{figure}[]
    \centering
         \includegraphics[width=\columnwidth]{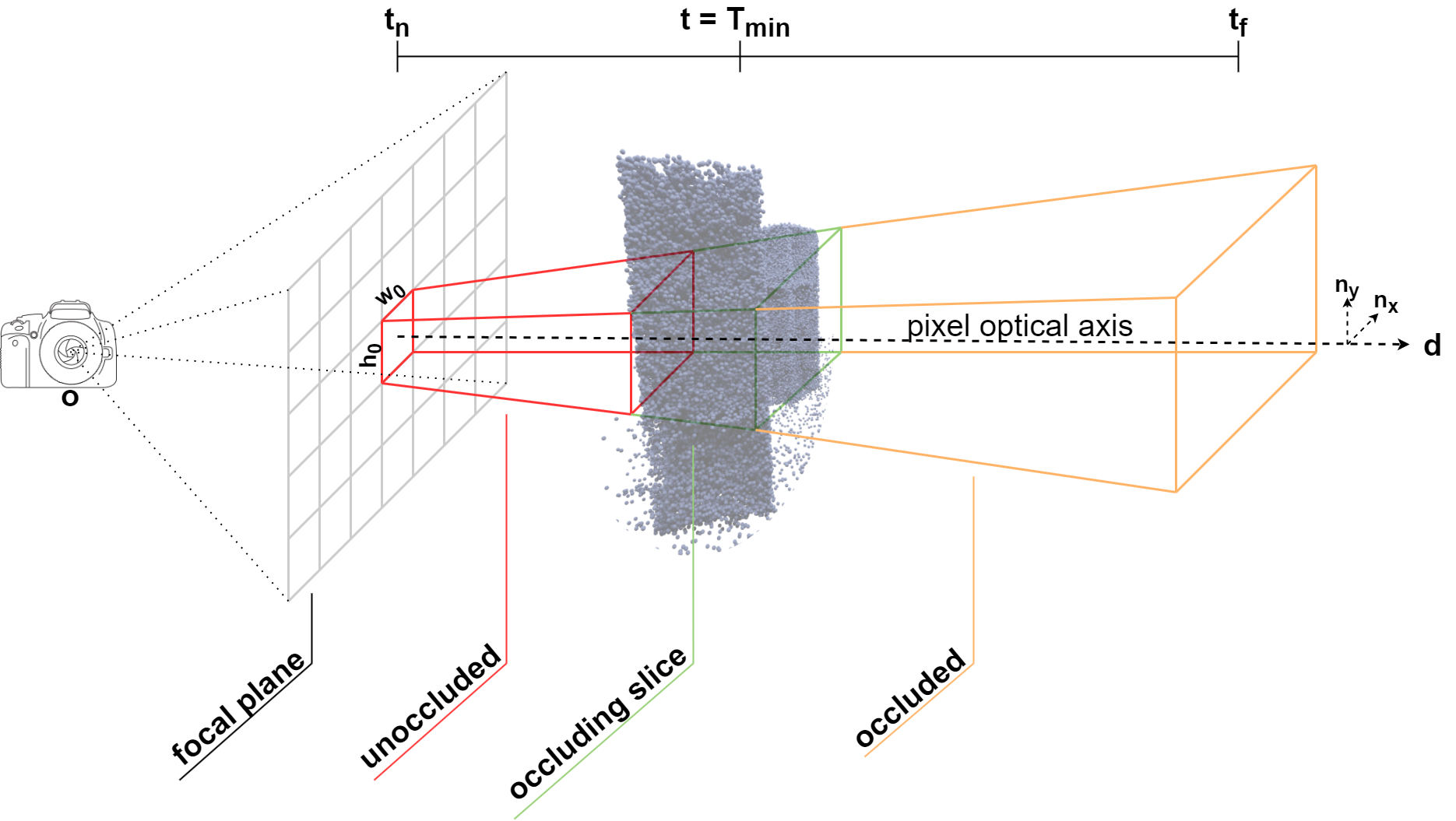}
        \caption{In the rendering process, the probability that the pixel color takes on the color of the infintesimally small occluding slice (green) is given by the probability that all slices in the region preceding  the slice (red) are unoccluded. Then, the pixel color is the expectation of the color taken by varying the position of the occluding slice along the ray.}
    \label{fig:rendering_explainer}
    \vspace{-2em}
\end{figure}

The intensity function of a PPP admits an infinitesimal interpretation: $\lambda(\mb{x}) \; d\mb{x}$ is the probability that a point of the process lies within an infinitesimal volume $d\mb{x}$ centered at $\mb{x}$. This is closely related to the interpretation of the density field offered by the original NeRF authors \cite{mildenhall2020a}: ``$\rho(\mb{x})$ defines the infinitesimal probability of a ray terminating at a given point $\mb{x} \in \R^3$.'' In this section, we show the volumetric rendering procedure introduced in \cite{mildenhall2020a} in fact computes an expected color of a marked PPP along a given ray.

Our key problem is how to relate the ray tracing used in NeRF, a 1D process, to the 3D measure of occupancy, $\lambda(\mb{x})$ in the PPP. We consider the occupancy swept by a frustum (the pyramid of light that projects onto a single pixel patch).  In particular, when generating the color of a particular pixel from a marked PPP, we extend a pyramid along the pixel's ray (Fig. \ref{fig:rendering_explainer}), and return the color of the first point of the process encountered along the ray in expectation. 

% To do this, we model ``rendering''  as beam tracing \cite{heckbert1984}, which considers the occupancy of ``thick beams'' (i.e., the pixel area swept along a particular ray) rather than the occupancy of infinitesimal rays. In particular, when generating the color of a particular pixel from a marked PPP, we extend a rectangular prism along the pixel's ray (Fig. \ref{fig:rendering_explainer}), and return the color of the first point of the process encountered along the ray. 

More specifically, for each pixel in the image, with associated ray $\mb{r}(\mb{o}, \mb{d})$, we consider the frustum (Fig. \ref{fig:rendering_explainer}, red) parameterized by a length $t \in [t_n, t_f],$

\begin{align}
    \mathcal{F}(t) \equiv \left\{\mb{o} + \mb{d}\tau + x \hat{\mb{n}}_x + y\hat{\mb{n}}_y \mid \lvert x \rvert \leq \frac{w(\tau)}{2}, \right. \nonumber\\ \left. \lvert y \rvert \leq \frac{h(\tau)}{2}, \tau \in [t_n, t] \right\}
\label{eq:beam_def}
\end{align}

where $\hat{\mb{n}}_x, \hat{\mb{n}}_y$ are unit vectors orthogonal to $\mb{d}$ forming a basis in the image plane and $h, w$ are the height and width of the frustrum cross-section, respectively, starting from the size of the pixel $(h_0, w_0)$ on the image plane. 

However, there still remains a modeling choice: given a fixed particle size $V_\text{ref}$, how many particles must be present in a given cross-section of the frustum for light to be occluded? We say the ray is occluded when the combined frontal area of the particles present at depth $t$ occupies a given fraction $\gamma \in (0, 1]$ of the frustum's cross-section.

Since the frustum's area $A(t)$ varies with depth, for particles of a fixed size, the number of particles needed to occlude the ray would also vary with $t$. However, as previously discussed, reweighting a PPP is equivalent to changing the particle size. Thus, in this work we consider ``dimensionless'' particles whose (projected) area on the frustum is exactly $\gamma A(t)$ (so only one particle is needed to occlude the ray) by reweighting the PPP density accordingly along the ray. Thus, we say the ray terminates at the depth of the first ``dimensionless'' particle encountered along the ray.

% We say the ray terminates wherever it first encounters a point of the process, i.e., at the first depth $t$ where the set $\mathcal{B}(t)$ is nonempty.

\subsection{Equivalence of NeRF and PPP Rendering}

This brings us to our main result. Here we show that, under appropriate assumptions on the distribution of the training rays and the spatial variation of the NeRF density and color, the color of a given pixel in an image rendered from a NeRF (\ref{Eq:RayTracing}) is exactly the expected color of the same pixel rendered from a PPP with (scaled) intensity equal to the NeRF density $\rho$, and marking equal to the NeRF radiance $\mb{c}$. We then provide intuition on the spatial relation between $\lambda$ and $\rho$. 

\begin{assumption}[PPP Smoothness]
\label{Ass:Smooth}

Consider a Poisson Point Process $\lambda(\mb{x})$ and color marking $\mathbf{c}(\mb{x},\mathbf{d})$. We assume the average of the PPP density over any ball $B_{\epsilon}$, with center $\mb{x_\epsilon}$, is equal to the value of the PPP density at the center of the ball, where $\epsilon$ is the minimum radius ball required to contain a pixel projected from the near plane onto the far plane of the NeRF scene. We also assume the color field $\mathbf{c}(\mb{x}, \mathbf{d})$ is approximately constant over any $B_\epsilon$.  Specifically, 
\begin{equation}
\label{Ass1:Density}
\int_{B_{\epsilon}} \lambda(\mb{x})\;  d\mb{x} = \lambda(\mb{x_\epsilon}) V_{B_{\epsilon}}
\end{equation}
and
\begin{equation}
\label{Ass1:Color}
\forall (\mb{x}, \mb{y}) \in B_{\epsilon}: \; ||\mathbf{c}(\mb{x}, \mathbf{d}) - \mathbf{c}(\mb{y}, \mathbf{d})|| = 0,
\end{equation}
where $B_\epsilon$ is the smallest ball that can contain an image pixel projected from the near plane onto the far plane of the NeRF scene.

\end{assumption}

This assumption states that the variation in the PPP over a small ball integrates to 0, while the color is constant in that same region.\footnote{Note that we do not require the density assumption for NeRF-PPP equivalence (Prop. \ref{Prop:ppp_rendering}, \ref{prop:nerfppp-equal}). Its use is in extracting an approximate scaling factor to transform between the density $\rho$ and PPP intensity $\lambda$ (Cor. \ref{corr:lambda2rho}).}  These are both mathematical idealizations which are unlikely to hold exactly in practice.  However, we find empirically that these assumptions are very close to being satisfied, ultimately yielding well-calibrated collision probabilities. For example, for a scene with length scale 2 meters, a camera with focal length 50 mm, and a $1000 \times 1000$ pixel image, the far plane pixel has side length on the order of $10^{-3}$m, requiring $\epsilon = \sqrt{2}$ mm.  Empirically, this is consistent with the smallest resolution of detail in a well-trained NeRF scene of a $2$ m$^3$ volume. 

Assumption \ref{Ass:Smooth} is reasonable due to the loss of information when encoding the continuous environment into the discretized observation space of pixelated images. Due to the resolution of the camera, color is constant across a pixel, hence we do not have information to distinguish generating environments whose colors only differ over the length scales of a single pixel. Since reconstructing the continuous geometry given the pixel-discretized images is ill-posed, we see this smoothness requirement as a kind of regularization prior for the density and color fields.  

Moreover, we tolerate stricter assumptions on the color because, in practice, the radiance field (i.e., color) is also smoother than the density field. The primary reason is that the radiance field is defined even in regions of empty space and is therefore allowed to smoothly change across surfaces; on the other hand, the density must change more sharply across surfaces in order to reflect the underlying discrete change in geometry. This is indeed reflected in the literature, where the original NeRF work \cite{mildenhall2020a} uses a smaller network to model color and later extensions like \cite{plenoxels} replace the network with smooth low-order spherical harmonics.

\begin{assumption}[Ray Redundancy]
\label{Ass:RayRedundant}
Given a dataset of training rays derived from  training images, poses, and camera intrinsics, no two rays intersect. 
\end{assumption}

This is a weak assumption as ray intersection is a zero-measure event. Moreover, floating point precision and noise in the poses further reduces the likelihood of ray intersection. %\textcolor{red}{Mac: I am commenting this out.  If we want to state e different assumption (e.g., no ray are colienar), then we state it.  It just creates confusion to state an assumption, and then say we actually could have stated a different weaker assumption.  Pick a line of argument and stick to it.} Finally, this assumption is not absolutely necessary to derive a PPP from the  NeRF, as we can choose to simply ignore the points of intersection while the integrals over the PPP or the NeRF density remain the same. In fact, we only require that rays not be co-linear.  

% \pc{This seems a little muddy -- maybe we should just make the assumption that rays aren't colinear? We could still include the remark about intersections being zero-measure.}

Given these assumptions, we now state our main result.

% \begin{proposition}[PPP-Density Equivalence]
% \label{Prop:ppp-density-equal}
% Consider a spatial density field $\rho(\mathbf{x})$ and a Poisson Point Process intensity field $\lambda(\mb{x})$. Under Assumption ~\ref{Ass:RayRedundant}, then there always exists an intensity field that describes the density field and vice versa. Specifically,

% \begin{align*}
% \int_{-\frac{x_0 + \delta_x}{2}}^{\frac{x_0 + \delta_x}{2}} \int_{-\frac{y_0 + \delta_y}{2}}^{\frac{y_0 + \delta_y}{2}} \lambda(\mb{r}(\tau) + x \hat{\mb{n}}_x + y \hat{\mb{n}}_y) \,dx\, dy\,  \equiv \rho(\mb{r}(\tau)) \; \forall \; \tau
% \end{align*}

% \end{proposition}

% \begin{proof}
% The existence of $\rho$ given $\lambda$ is immediate by definition. Given $\rho$, we seek to find a suitable $\lambda$. Taking the second derivative of $\rho$ and applying the Leibniz integral rule, we retrieve 

% \begin{align*}
% \sum_{i=0}^1 \sum_{j=0}^1 (-1)^{i+j}\lambda(\mb{r}(\tau) + (-1)^i \frac{x_0 +  \delta_x}{2}\hat{\mb{n}}_x + (-1)^j \frac{y_0 + \delta_y}{2} \hat{\mb{n}}_y) \\
% \equiv \frac{\partial^2}{\partial x_0 \partial y_0}\rho(\mb{r}(\tau)) \; \forall \; \tau
% \end{align*}

% Given Assumption \ref{Ass:RayRedundant}, for every $(x_0, y_0)$, there is only one set of $(\delta_x, \delta_y)$. Furthermore, the vertices of a cross-section of a ray will almost surely not intersect vertices of any other cross-section on any ray. Since there are many more degrees of freedom to fitting $\lambda$ to $\rho$, there exists an intensity field that describes the density exactly. 

% \end{proof}

\begin{proposition}[Rendering of PPP]
\label{Prop:ppp_rendering}
Consider a PPP $\lambda(\mb{x})$ and radiance $\mathbf{c}(\mb{x},\mathbf{d})$ and let the radiance satisfy (\ref{Ass1:Color}) from Assumption~\ref{Ass:Smooth}. Then, the expected color of a pixel rendered from the PPP matches the form of the rendering equation (\ref{Eq:RayTracing}).
\end{proposition}

\begin{proof}
Let us consider a ray $\mathbf{r}(t) = \mathbf{o} + t \cdot \mathbf{d},$ where $t \in [t_n, t_f]$. We consider again the corresponding frustum $\mathcal{F}(t)$ \eqref{eq:beam_def}, the pyramid created by sweeping the scaled pixel area along $\mb{r}$ from $t_n$ to $t$.  

% Crucially, we need to define occlusion for each cross-sectional slice of the frustum orthogonal to the optical axis. To be exact, the binary event of occlusion occurs when at least 1 particle the size of some fraction $\gamma$ of the cross-sectional slice is present in that slice. We call $\gamma$ the \emph{occlusion threshold} (the $\%$ of the slice that needs to be occupied to be considered occluded). Thus, for slices farther along the ray, we consider larger particles since its size is proportional to the slice's area. 

As discussed in our rendering model, we consider ``dimensionless'' particles whose projected area is $\gamma A(t)$ (so the ray is occluded by the first particle encountered). Thus, the intensity/expected number of particles of this slice $\delta \mathcal{F}(t)$ for those of the occluding size ($V_d = \gamma V_{\delta \mathcal{F}}(t)$) is defined as

% \pc{So is the event that a particle big enough to occlude is present in the slice? Or if the total frontal area of all particles is enough to occlude? If it's the first, it seems like you could never get the ray to terminate after the frustum became too large, since the particle size is fixed (?). Tim: Particle size is not fixed, that's where the scaling factor $V_{\delta B}$ comes from. }

\begin{align*}
    &\Lambda\big(\delta \mathcal{F}(t)\big) = \int_{\mb{x} \in \delta \mathcal{F}(t)} \frac{V_{ref} \lambda(\mb{x})}{\gamma V_{\delta \mathcal{F}}(t)}\;d\mb{x} \\
    &= \int_{t-\delta t/2}^{t + \delta t/2}\int_{(x, y) \in A(\tau)} \frac{V_{ref} \lambda(\mb{r}(\tau) + x\hat{\mb{n}}_x + y\hat{\mb{n}}_y)}{\gamma V_{\delta \mathcal{F}}(t)} \, dx\, dy\, d\tau\\
    &= \frac{\delta t \int_{A(t)} A_{ref} \delta t \lambda(\mb{r}(t) + x\hat{\mb{n}}_x + y\hat{\mb{n}}_y) \, dx\, dy\,}{\gamma A(t) \delta t}
\end{align*}
% \pc{nit: needs reformatting}
if the reference particle is some small ball of size $V_{ref} = A_{ref} \delta t$.

Hence, the void probability of the slice (\ref{eq:void}) (equivalently, the probability of no occlusion) is \begin{align*}
&Pr(N(\delta \mathcal{F}(t)) = 0) = \\&\exp\left[- \frac{\int_{A(t)} A_{ref} \lambda(\mb{r}(t) + x\hat{\mb{n}}_x + y\hat{\mb{n}}_y) \, dx\, dy\,}{\gamma A(t)} \delta t\right].
\end{align*}

Now we consider the event where any slice of the frustum is occluded, up to a given depth $t$. To do this, we divide the frustum into smaller subsections along its length. Because the number of particles in disjoint subsets are independent by definition of the PPP (Def. \ref{def:ippp}(ii)), then the probability of occlusion of each section is independent of that of other sections. Hence, the probability that the frustum up to $t$ is not occluded requires all sections to not be occluded,\begin{align*}
&\prod_t Pr(N(\delta \mathcal{F}(t)) = 0) \\& =\prod_t \exp\left[- \frac{\int_{A(t)} A_{ref} \lambda(\mb{r}(t) + x\hat{\mb{n}}_x + y\hat{\mb{n}}_y) \, dx\, dy\,}{\gamma A(t)} \delta t\right]\\
&=\exp\left[-\sum_t\frac{\int_{A(t)} A_{ref} \lambda(\mb{r}(t) + x\hat{\mb{n}}_x + y\hat{\mb{n}}_y) \, dx\, dy\, }{\gamma A(t)} \delta t\right].
\end{align*}

In the limit as the section widths $\delta t$ approach zero, they become slices, and the summation becomes an integral. Thus, the probability that the frustum up to $t$ is occluded is \begin{align*}
&Pr(\mathcal{F}(t)\,  \text{not occluded})\\
&=\exp\left[-\int_t\frac{\int_{A(\tau)} A_{ref} \lambda(\mb{r}(\tau) + x\hat{\mb{n}}_x + y\hat{\mb{n}}_y) \, dx\, dy\,}{\gamma A(\tau)} d\tau\right] \, .
\end{align*}

For notational simplicity, we henceforth denote the surface integral divided by the scaled area of the slice by $\kappa(\tau)$. 

Let us now consider a random variable $T_\text{min}$ which defines the distance of the first occluding slice (denoted green in Fig. \ref{fig:rendering_explainer}). We can define the cumulative distribution function of this variable, for some $t > t_n$, $T_\text{min} \leq t$ as the probability that $\mathcal{F}(t)$ is occluded.  %$N\big(\mathcal{B}(t)\big) > 0$. %From the definition of the PPP, we can compute the void probability \
% \begin{align}
% P\left(N\big(B(t)\big) = 0\right) &= \frac{\Lambda\left(B(t)\right)^0}{0!} \exp\left[-\Lambda\big(B(t)\big)\right] \\
% &= \exp\left[-\Lambda\big(B(t)\big)\right].
% \end{align}
Thus, the CDF of $T_\text{min}$ can be defined using the above equation,
\begin{align*}
Pr(T_\text{min}\leq t) \equiv F_{T_\text{min}}(t) &= 1 - P\left(\mathcal{F}(t)\,  \text{not occluded}\right)\\
 % &= 1 - \exp\left[-\Lambda\big(\mathcal{B}(t)\big)\right],\\
 & = 1 - \exp\left[-\int_{t_n}^t  \kappa(\tau) \, d\tau \right].
\end{align*}

We can then compute the PDF of $T_\text{min}$ by differentiating $F_{T_\text{min}}$ with respect to $t$, yielding,
\begin{equation}
\begin{split}
     f_{T_\text{min}}(t) &= \frac{d}{dt} F_{T_\text{min}}(t) \\
    &=\kappa(t) \exp\left[-\int_{t_n}^t  \kappa(\tau) \, d\tau\right]. 
    \end{split}
\label{eq:depth_likelihood}
\end{equation}
This defines a probability distribution over the extent of the unoccluded region.

Finally,  due to Assumption~\ref{Ass:Smooth}, the color of the slice returned by the PPP rendering is equal to the marking function evaluated at $\mb{r}(T_\text{min}).$ Thus, we can compute the expected color of the PPP rendering by computing the expectation of $\mathbf{c}(\mathbf{r}(T_\text{min}))$, yielding
\begin{align*}
    \mb{C}(\mathbf{r}) &= \mathbb{E}\left[ \mathbf{c}(\mathbf{r}(T_\text{min})) \right],\\
    &= \int_{t_n}^{t_f} \kappa(t) \mathbf{c}(\mathbf{r}(t)) \exp\left[-\int_{t_n}^t  \kappa(\tau) \, d\tau \right] \; dt.
\end{align*}
The PPP expected color matches exactly the expression given in (\ref{Eq:RayTracing}) from the original NeRF paper \cite{mildenhall2020a}, completing the proof. 
\end{proof}

\begin{proposition}[NeRF-PPP Equivalence]
\label{prop:nerfppp-equal}
Consider a NeRF with density $\rho(\mb{x})$ and let Assumption \ref{Ass:RayRedundant} hold. Then the NeRF is a locally area-averaged PPP. 

% Let $C(\mb{o}, \mb{d})$ be the color of a pixel of width $w_0$ and height $h_0$ rendered from the NeRF \eqref{Eq:RayTracing}. Then, the pixel color $C(\mb{o}, \mb{d})$ of the NeRF is the expected color of a pixel rendered from a PPP with intensity $\lambda(\mb{p}) = f(\rho(\mb{p}), \mb{p})$ and marking $\mb{c}(\mb{p}, \mb{d}).$
\end{proposition}

\begin{proof}
Following from the above proof, we make equivalences between the two rendering equations for all training rays, namely 

\begin{align*}
\forall \; \mb{r}(t; \mb{o}, \mb{d}):& \, \\
\rho(\mb{r}(t)) &= \frac{\int_{A(t)} A_{ref} \lambda(\mb{r}(t) + x\hat{\mb{n}}_x + y\hat{\mb{n}}_y) \, dx\, dy\,}{\gamma A(t)}.
\end{align*}

Note that we require Assumption \ref{Ass:RayRedundant} because when two rays intersect, the left hand side of the above equation is necessarily identical for both rays, yet the right hand side may not be. Specifically, the integration domains for the two rays at the intersection point may not be identical, hence the numerator and denominator on the right hand side are not the same for both rays. Moreover, note that for $\gamma = 1$ (full occlusion), the density is the area-averaged number of particles over the slice $A(t)$.
\end{proof}

Crucially, this matches the intuition on the NeRF density proposed by \cite{max-optical, mildenhall2020a}. However, our derivation is more general than that of \cite{max-optical}, which assumes a constant-area frustum. Although we used a pyramidal frustum for illustration purposes, note that our derivation does not assume the form of $A(t)$ (e.g. rectangular, circular) and therefore the shape of the frustum, so long as the frustum can be decomposed into disjoint slices that are themselves connected sets. Finally, our derivation suggests a more general rendering equation since we had to assume local homogeneity of color to retrieve (\ref{Eq:RayTracing}). Without this limitation, we could derive a more expressive and accurate rendering equation. Moreover, our derivation even proposes a parameter $\gamma$ that can be tuned to more precisely define occlusion and perhaps increase the fidelity of the render. 

\begin{corollary}
\label{corr:lambda2rho}
By (\ref{Ass1:Density}) from Assumption \ref{Ass:Smooth}, $\rho(x) = A_{ref} \frac{\lambda(x)}{\gamma}$, where $0 < \gamma \leq 1$, so the NeRF density is related to an equivalent PPP through a constant scale factor $A_{ref}/\gamma$. 
\end{corollary}

In general, the constant $A_{ref}/\gamma$ is unknown, however we show in Sec.~\ref{Sec:CollisionProbability} below that we do not need its exact value to compute the collision probability for a robot body.

Having shown that a PPP can be derived from the NeRF, we visually show this relationship in Fig.~\ref{fig:ppp-st-overlay}, where we generate a point cloud randomly drawn from the PPP (blue spheres) superimposed on the NeRF rendering of the same scene.  The correspondence in geometry of the NeRF scene and the PPP point cloud is clear. 

\begin{figure}[]
     \centering
        \includegraphics[width=0.49\textwidth]{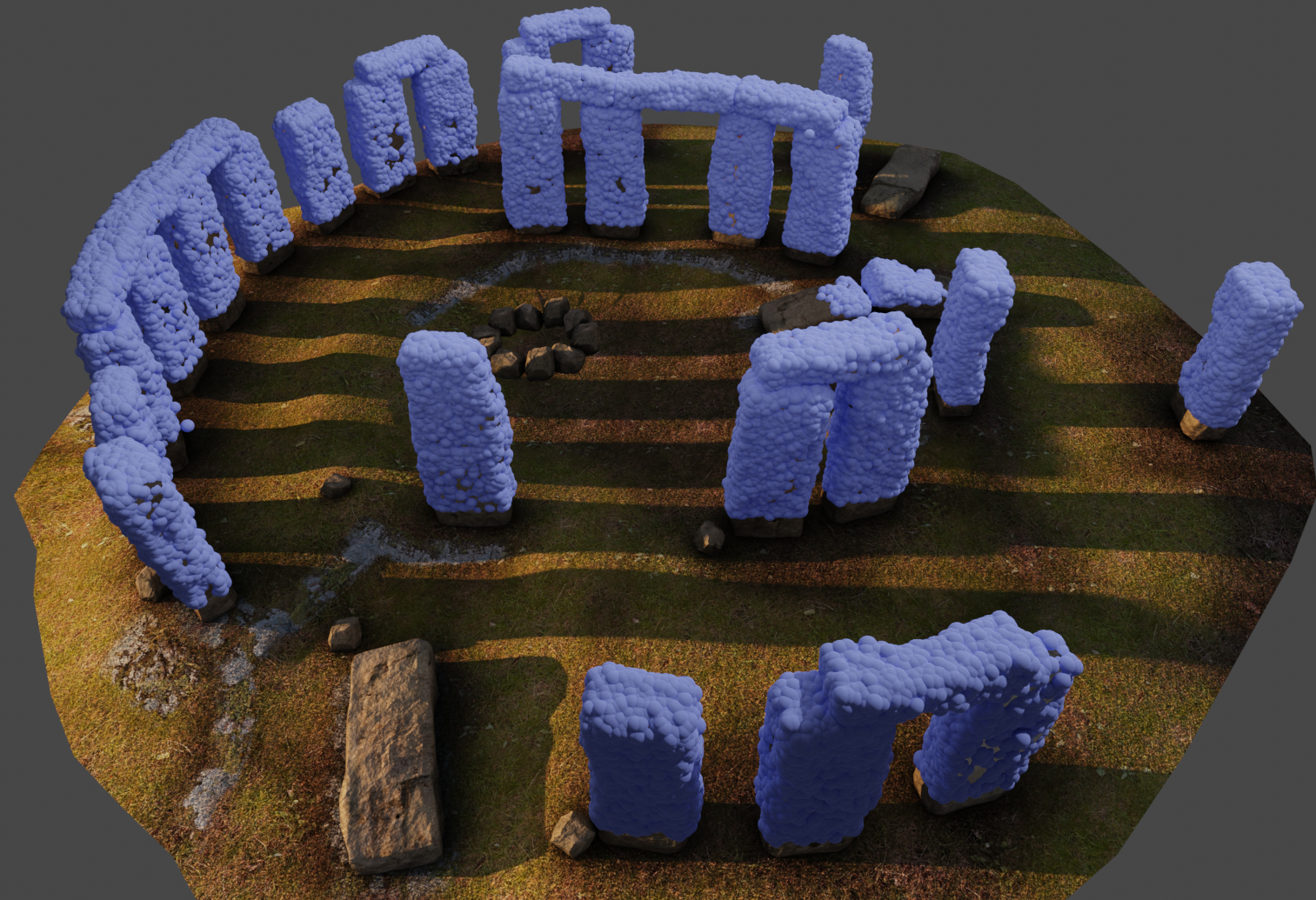}
        \caption{Overlay of a realization of the PPP with the ground-truth mesh of Stonehenge. The two have strong spatial agreement.}
        \label{fig:ppp-st-overlay}
\end{figure}
Note that while area-averaging of $\lambda$ to $\rho$ yielded the rendered color as a line integral as opposed to a volume integral, we have lost information about the $\lambda$ field (i.e. the reference particle $A_{ref}$ and occlusion $\gamma$) when using a learning framework to learn $\rho$. In fact, we conjecture the observed aliasing phenomena in which NeRFs fail at different scales \cite{mipnerf} is due to this averaging scheme. The success of works like Mip-NeRF \cite{mipnerf} that reason about the pixel not as a projection of a ray, but along a frustrum and evaluating rendering as a volume integral (i.e., learning $\lambda$ rather than $\rho$) gives us reason to believe that learning the parameters of the PPP directly could yield higher quality geometry.

\subsection{Discussion of PPP Interpretation}

An important advantage of taking a PPP interpretation of the NeRF density is that it allows us to leverage well-studied properties of PPPs when reasoning about NeRFs. Specifically, uncertainty metrics such as likelihood and entropy are well-defined for PPPs. Suppose we measure a point cloud of our environment using an onboard lidar or depth camera; i.e., for a set of rays $\{\mb{r}_1, \ldots, \mb{r}_k\}$ we obtain noiseless depth measurements $\{d_1, \ldots d_k\}$. We can write the likelihood of obtaining these depths under our NeRF using \eqref{eq:depth_likelihood},

\begin{equation} \label{ppp-ll}
\log P(d_1, \ldots d_k) = \sum_{i=0}^k \left[\log \rho(\mb{r}(d_i)) - \int_{t_n}^{d_i} \rho(\mb{r}(t))dt\right].
% \begin{split}
% \log(L) = \sum_{i=0}^{N}\log(\lambda(X_i)) - \int_B \lambda(x)\; dx
% \end{split}
\end{equation}

If, say, the robot's state is uncertain, and the depth measurements are corrupted by noise, this likelihood can be used in the computation of Bayes' rule for pose estimation. Since previous literature on NeRFs contained no such likelihood interpretation, existing approaches to state estimation  \cite{yen2021inerf, adamkiewicz2022vision} instead only minimize a photometric loss as a proxy for maximum likelihood estimation.

Further, notions of entropy and mutual information can be generalized to point processes, as discussed in \cite{ppp-entropy}. In particular, the entropy of a point process over a set $B$ is defined as
\begin{align}
\mathbb{H}(B) &\equiv \int_B \lambda(\mb{x}) (1 - \log \lambda(\mb{x}))d \mb{x}\\
&= \frac{\gamma}{A_{ref}} \int_B \rho(\mb{x}) (1 + \log A_{ref} - \log \gamma - \log \rho(\mb{x})) d\mb{x} \label{ppp-entropy}.
\end{align}
Thus, $\mathbb{H}$ can be used as a measure of the uncertainty of a NeRF over some set $B$, which would be useful to reason about which parts of the NeRF are poorly-supervised (i.e., have high entropy) for problems such as next-best-view selection and active perception. 

Finally, the PPP interpretation of the NeRF allows us to provide a novel perspective on NeRF training: minimizing the photometric loss \eqref{Eq:PhotometricLoss} proposed in \cite{mildenhall2020a} can be interpreted as performing moment-matching \cite[Ch.~4]{vaart_1998} on the first moment of the color distribution along the supervised rays. In particular, the photometric loss will be zero if the color rendered from the NeRF (i.e., an expected color along the ray of a PPP) matches the sample distribution (i.e., the color label in the dataset). We believe our probabilistic interpretation of the NeRF density could also inspire other loss functions beyond \eqref{Eq:PhotometricLoss}, to perform other methods of parameter estimation such as maximum likelihood estimation, expectation maximization, and so on.

\section{Computing Collision Probability with NeRF Scenes}
\label{Sec:CollisionProbability}

% \subsection{Collision Probability with a Robot Body}\label{coll_prob}

To leverage the probabilistic interpretation of NeRFs to evaluate the probability of collision between a robot body and the NeRF, we first define $B(\mathbf{p},\mb{R}) \subset \mathbb{R}^n$ as the robot body parameterized by its pose $(\mathbf{p}, \mb{R})$ (the set of points occupied by the robot with position $\textbf{p}\in\mathbb{R}^3$ and orientation $\mb{R}\in \mathbb{SO}(3)$), and consider an environment represented as a NeRF with density field $\rho(\mathbf{p})$, which we have shown to be related to the PPP field through a constant scale factor $\lambda(\mathbf{p}) = \frac{\gamma \rho(\mathbf{p})}{A_{ref}}$.

We define collision probability, or the probability that a collection of points from a PPP intersects with the robot body, as the probability that at most some volume $V_{max}$ from the NeRF can exist within the robot volume. We call $V_{max}$ the specified or allowable inter-pentration volume. Given some auxiliary particle (which may not be the same as the reference particle) that is user-defined and has some volume $V_{aux} < V_{max}$, we can solve for the maximum number of auxiliary particles that should exist in the robot volume $N_{aux}^{max} = \frac{V_{max}}{V_{aux}}$. The definition of this new type of particle is necessary because one does not have access to the reference particle dimensions of the underlying PPP. We show that this knowledge is not necessary to compute collision. Nonetheless, we are simply solving for the Cumulative Distribution Function (CDF) of the Poisson Point Process associated with the auxiliary particle (i.e. the number of particles that exist in the robot body), 

\begin{equation}\label{chance_constraint}
    Pr(X \leq N_{aux}^{max}; \Lambda_B) = \exp^{-\Lambda_B}\sum_{i=0}^{\lfloor {N_{aux}^{max}}  \rfloor}\frac{\Lambda_B^i}{i!},
\end{equation}
where $\Lambda_B$ is the intensity for the auxiliary particle over the robot body. 

Note that $\Lambda_B$ is the PPP associated with the auxiliary particle and not the reference; therefore, we must scale the NeRF density appropriately, 
\begin{equation}
    \Lambda_B = \int_B \lambda_{aux}(\mb{x})\; d\mb{x} = \frac{V_{ref}}{V_{aux}} \int_B \lambda(\mb{x})\; d\mb{x}.
\end{equation}

Recall that $V_{ref} = A_{ref} \delta t$ and assume an identical form for the auxiliary volume $V_{aux} = A_{aux} \delta t$. Additionally, we can substitute the relationship between $\lambda$ and $\rho$ (Cor. \ref{corr:lambda2rho}) to get the following
\begin{equation}
\label{eq:robot_intensity}
    \Lambda_B = \frac{A_{ref} \delta t}{A_{aux} \delta t} \int_B \frac{\gamma \rho(\mb{x})}{A_{ref}} \; d\mb{x} = \frac{\gamma}{A_{aux}} \int_B \rho(\mb{x})\; d\mb{x}.
\end{equation}
This brings us to the first formal definition of NeRF collision.

\begin{definition}[Probabilisitically Safe]
A robot body parametrized by its pose $B(\mathbf{p}, \mb{R})$ is probabilistically safe if the collision probability given in (\ref{chance_constraint}) and (\ref{eq:robot_intensity})
 satisfies $Pr(N(B(\mathbf{p}, \mb{R})) \leq N_{aux}^{max}; \Lambda_B) \geq \sigma$, for desired probability threshold $\sigma$, auxiliary particle associated with the specified inter-penetration volume, and occlusion threshold.
\end{definition}

More succinctly, we consider a robot as probabilistically safe if the interpenetration volume between the robot and the NeRF is less that a threshold $V_{max}$ with probability at least $\sigma$. A major remaining question is how to choose the auxiliary particle size and occlusion threshold $\gamma$. Since we are assuming smoothness on the length scales of a pixel, it is appropriate to use an auxiliary particle of this size. We use the approximate size of a pixel at the near plane and the sampling distance along a ray to find the dimensions of the auxiliary particle. Specifically, given a camera with focal length 50 mm, a scene of length 2 m, a FOV of $90^\circ$, and a 1000 by 1000 image, the pixel side on the image plane is $10^{-4}\; \text{m}$, hence for a square pixel, $A_{aux} = 10^{-8} \; \text{m}^2$. If the pixel size is the 2D resolution of the image, we can think of $\delta t$ as the sampling resolution used to learn the NeRF, hence $V_{aux}$ the 3D resolution of the NeRF learned from data. Typically, there are anywhere between 100 to 200 sample points along a ray between the near and far plane, so we choose $\delta t = 20 \; \text{mm}$, yielding $V_{aux} = 2 \cdot 10^{-8} \; \text{m}^3$. In fact, the CDF is relatively insensitive to $A_{aux}$, as both $N_{aux}^{max}$ and $\Lambda_B$ are scaled the same amount for changes in $A_{aux}$. $\gamma$ is essentially a modelling parameter as there is no way to know what the environment defines as an occlusion event. However, in the interest of safety and interpretability, we set $\gamma = 1$ so that it is meaningful (i.e. full occlusion) and such that it yields the most conservative estimate for $\lambda$.

\section{Chance-Constrained Trajectory Generation in NeRFs}
\label{Sec:CollisionAvoidance}

We now consider the problem of real-time trajectory planning for a robot in an environment represented as a NeRF, subject to constraints on the probability of collision \eqref{chance_constraint}. Our proposed algorithm, CATNIPS, has two parts: we first generate a lightweight, voxel-based scene representation, which we term a  Probabilistically Unsafe Robot Region (PURR), that encodes the robot locations that satisfy the collision constraint for a particular robot geometry. We then use Bézier curves to plan safe trajectories in position space for a robot traversing the PURR subject to the probabilistic collision constraint. By assuming differential flatness of our robot, we can guarantee dynamic feasibility of our solution when planning in the flat output space.

We note that the following collision calculations may be conservative due to the approximation of the robot as a sphere such that safety is invariant to orientation. However, the collision metric (\ref{chance_constraint}) could be calculated while taking into account the position, orientation, and the physical geometry of the robot. We also note that the differentiability of the collision probability permits its use in gradient-based trajectory optimizers. Nonetheless, our design choices were motivated by speed, scalability, and modularity with other downstream tasks (e.g. active view planning where the orientation can be freely manipulated).

\subsection{Generating the PURR}
The PURR is a binary, voxelized representation of the NeRF that indicates collision in $\mathbb{R}^3$. If the robot's position is located within a free voxel in the PURR, the robot is probabilistically safe, i.e., the chance collision constraint is guaranteed to be satisfied.  Otherwise, safety is not guaranteed---the robot may (or may not) be in violation of the chance constraint.   %The space of the scene is first gridded into voxels indexed as $v_{ijk}$. If a robot's position $\mathbf{p}$ is located in any of the occupied voxels of the PURR, then probabilistic safety over the entire robot body may be violated. Conversely, existence in any voxel not tagged is guaranteed safe. 

Similar to classical configuration space planning \cite{lozano1990spatial}, the core idea behind the PURR is to inflate the occupied space in the map such that planning a path through free space in the inflated map corresponds to a robot trajectory that satisfies the collision probability constraint in the underlying NeRF map.  We essentially ``inflate'' the NeRF density function to account for both the robot geometry and the chance constraint.  

Fig.~\ref{fig:purr_explainer} shows a schematic of the PURR generation process, as well as how the PURR fits into our trajectory planning pipeline. We first voxelize the space of the map and label each cell with the integral of the NeRF density multiplied with a scaling factor over each cell to give the voxelized \emph{cell} intensity grid, $\mathbb{I}_c$. We then take the Minkowski sum between a sphere bounding the robot and one underlying voxel cell to produce the set of all voxels that the robot could be touching, in any orientation, if its center of mass were located anywhere in one cell;  we call this the robot kernel, $\mathbb{K}$. Finally, we convolve the robot kernel $\mathbb{K}$ with the cell intensity grid $\mathbb{I}_c$ to obtain the \emph{robot} intensity grid, $\mathbb{I}_r$. Finally, we evaluate the Poisson CDF using the robot intensity (as well as the auxiliary particle parameters) and threshold the robot intensity grid with the user-defined collision probability threshold $\sigma$ to get the binary PURR map.  These operations are described in more mathematical detail as follows.

%the PPP occupancy use the PPP interpretation of the NeRF density to compute a voxelized map of the occupancy probability voxel compute through the  requires the creation of three objects: the \emph{cell} intensity grid, the robot bounding sphere, and the robot kernel (Fig. \ref{fig:purr_explainer}). The sphere is used to create the kernel, which is then convolved with the \emph{cell} intensity grid to create the \emph{robot} intensity grid. Each voxel value of the \emph{robot} intensity grid is an upper bound on the collision probability (\ref{chance_constraint}) for a robot body with position $\mathbf{p} \in v_{ijk}$ anywhere in voxel $v_{ijk}$. We threshold the grid by the desired probability cutoff $\sigma$, such that cells less than $\sigma$ are guaranteed to satisfy $Pr(N(B(\mathcal{T})) > 0) \leq \sigma$ (i.e. guaranteed to be probabilistically safe). Cells greater than $\sigma$ may potentially be unsafe and should be avoided. 

\begin{figure}[t]
         \centering
         \includegraphics[width=\columnwidth]{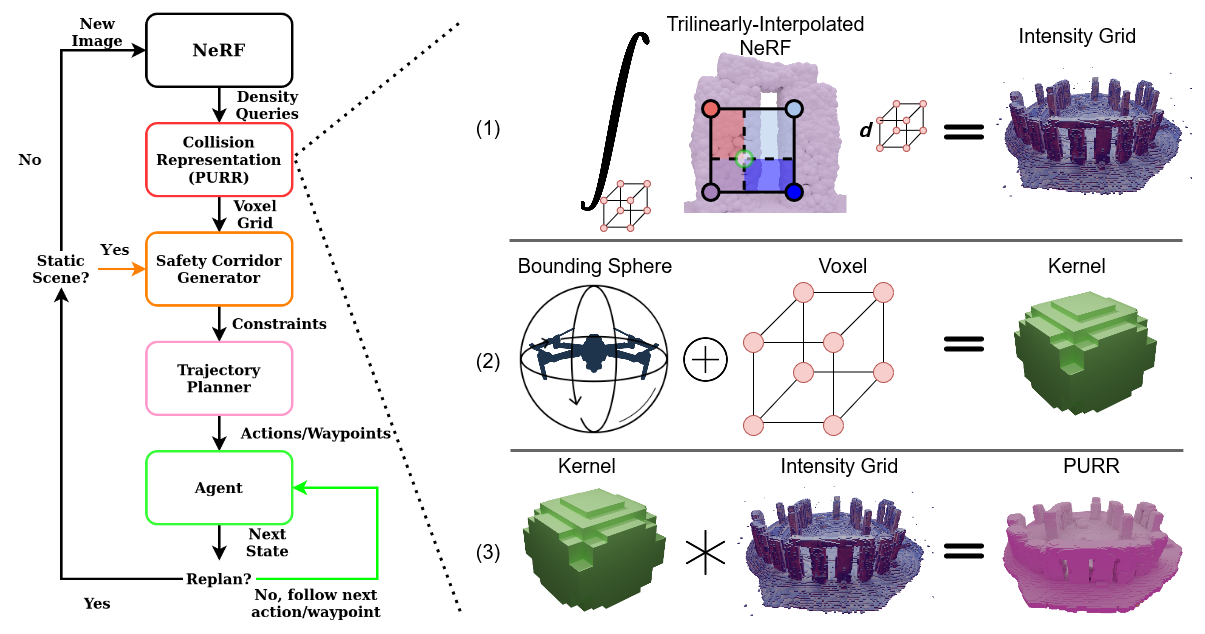}
         \caption{NeRF to PURR pipeline. (1) A density grid is sampled from the NeRF, which is then trilinearly interpolated and integrated over a particular voxel to retrieve the cell intensity grid. (2) A robot kernel is generated by taking the Minkowski sum between the minimal bounding sphere of the robot and a single voxel. (3) The kernel is used in a \texttt{Conv3D} operation with the cell intensity voxel grid to create the robot intensity grid, which we then threshold by the user-defined collision probability $\sigma$ to create the PURR.}
         \label{fig:purr_explainer}
\end{figure}

\subsubsection{Cell Intensity Grid}
The \emph{cell} intensity grid $\mathcal{I}_c$ computes, for each grid cell, the expected number of auxiliary particles in voxel $v_{ijk}$
\begin{equation}
\label{Eq:CellInt}
\mathbb{I}_c(v_{ijk}) = \int_{v_{ijk}} \frac{\gamma \rho(x)}{A_{aux}} dx.
\end{equation} This integral, in general, cannot be computed analytically since the density $\rho$ is typically represented using a neural network.  We compute a high-quality approximation of this integral using a trilinear interpolation scheme.  In fact, if the NeRF density uses an underlying voxel-based representation (as do the most high-speed and high quality NeRF variants in the literature \cite{mueller2022instant,plenoxels}) our integral of the trilinear interpolation is exact. 
 %Inspired by the speed and accuracy of voxel-based NeRFs \cite{plenoxels}, rather than computing the robot collision probability \eqref{chance_constraint} online using, say, Monte Carlo integration, we choose to ``bake'' the NeRF-based collision probabilities into a voxelized representation. 

We first discretize the environment spatially, using a rectangular grid and query the NeRF for the density values at the grid vertices. We then represent the continuous density field using trilinear interpolation as \cite{max-optical},
\begin{equation}\label{trilinear}
\begin{split}
\hat{\rho}(x, y, z) = &  c_1 + c_2 x + c_3 y + c_4 z + \\
& c_5 xy + c_6 yz + c_7 xz + c_8 xyz.
\end{split}
\end{equation}
The coefficients $\mb{c}_{1:8}$ are the solution of a linear system $\mathbf{A}_{1:8} \mb{c_{1:8}} = \mb{\rho_{1:8}}$, where $\mathbf{A}_{1:8}$ is the matrix of stacked row vectors of the terms involving cell vertex locations, and  $\mb{\rho_{1:8}}$ the densities at the vertices. Note that at the vertices of the cell, $\hat{\rho}(x, y, z) = \rho(x, y, z)$. Different interpolations exist for other finite element geometries, although we only consider rectilinear cells in this work. The cell values of the \emph{cell} intensity grid are computed from a closed form analytic solution to the  integral (\ref{Eq:CellInt}) plugging in (\ref{trilinear}) for $\rho$ over the extent of the cell. The analytic expression is given in Appendix \ref{app:trilinear_analytic}. 

\subsubsection{Robot Kernel}
The robot kernel $\mathbb{K}(v_{ijk})$ is a mask that indicates the neighborhood of cells around voxel $v_{ijk}$ that are considered in the computation of the collision probability when the robot position $\mathbf{p}$ is anywhere in $v_{ijk}$. We first find the Minkowski sum of the minimum bounding sphere\footnote{Inflating the robot body to a sphere removes the effects of robot orientation on safety. As a result, the PURR can be efficiently created in position space, while downstream tasks have the ability control the robot orientation without impacting safety.}
of the robot with the cell in which the robot center is located.  We then compute the smallest collection of voxels that contains this Minkowski sum.  This collection of voxels is the robot kernel, which can be efficiently convolved with the 3D voxelized grid using standard PyTorch functions.

\subsubsection{Robot Intensity Grid}
Once the robot kernel is defined, computing the collision probability for the robot in a particular cell simply requires a convolution between the kernel $\mathbb{K}$ and the cell intensity grid $\mathbb{I}_c$. In particular, we generate a \emph{robot} intensity grid $\mathbb{I}_r$, by convolving the kernel with each cell, giving the expected number of auxiliary particles in the body as follows,
\begin{equation}
\begin{split}
\label{robo_intensity_grid}
    \mathbb{I}_r (v_{ijk}) = & \sum_{v_{lmn} \in \mathbb{K}(v_{ijk})} \mathbb{I}_c (v_{lmn}) \\
     = &\sum_{v_{lmn} \in \mathbb{K}(v_{ijk})} \int_{v_{lmn}} \frac{\gamma \hat{\rho}(\mb{x})}{A_{aux}} \; d\mb{x}\\
     \geq & \int_{B(\mathbf{p}, \mb{R})} \frac{\gamma \hat{\rho}(\mb{x})}{A_{aux}}\;  d\mb{x} \; \forall \; \mb{R} \in \mathbb{SO}(3), \mathbf{p} \in v_{ijk}\\
     \approx & \int_{B(\mathbf{p}, \mb{R})} \frac{\gamma \rho(\mb{x})}{A_{aux}}\; d\mb{x} \; \forall \; \mb{R} \in \mathbb{SO}(3), \mathbf{p} \in v_{ijk}.
\end{split}
\end{equation}

Finally, the PURR $\mathbb{P}$ is generated by calculating the CDF (\ref{chance_constraint}) using the \emph{robot} intensity grid and thresholding by the collision probability threshold $\sigma$,
\begin{equation}
\label{purr_eqn}
    \mathbb{P}(v_{ijk}) = Pr(X \leq N_{aux}^{max}; \mathbb{I}_r(v_{ijk})) < \sigma.
\end{equation}
The resulting PURR is visualized at different specified inter-penetration volumes at a reasonable probability of $\sigma = 95\%$ in Fig.~\ref{fig:purr-viz} (Top) for the Flightroom NeRF environment.  On the bottom of Fig.~\ref{fig:purr-viz}, a simple density thresholded grid can yield visually similar voxel maps, but can degenerate arbitrarily quickly for large density cutoffs. This is because the threshold for the PURR is calibrated to a precise probability of collision, while thresholding on the density provides no interpretable safety metrics.

\begin{figure}[b]
         \centering         
         \includegraphics[width=\columnwidth]{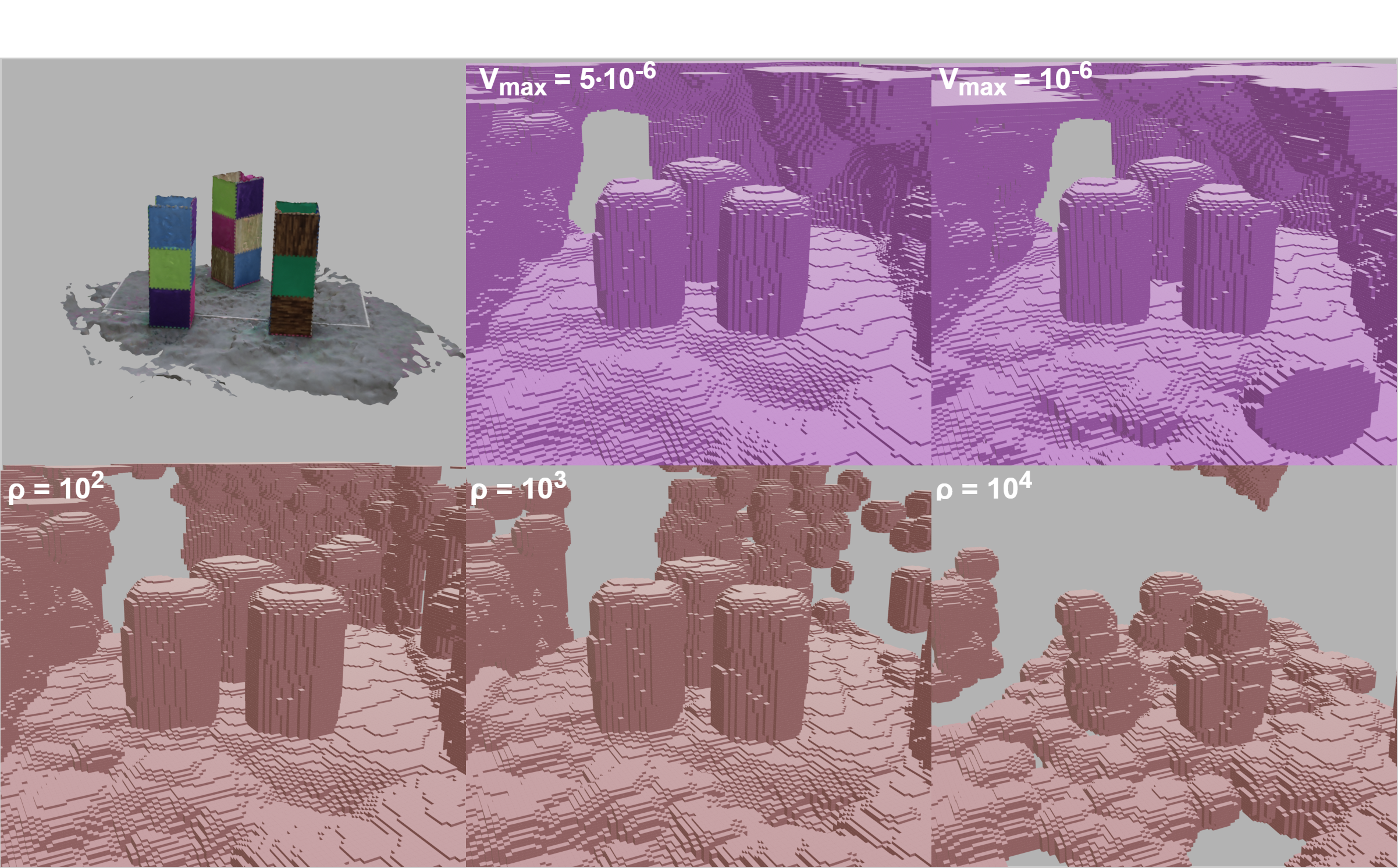}
         \caption{Top: Mesh of Flightroom, and PURR with varying specified inter-penetration values $V_{max}$ at fixed probability $\sigma = 95\%$. Bottom: Density-thresholded voxel maps with varying density values $\rho$. The PURR is precisely calibrated to give a desired probability of collision, while thresholding the NeRF density directly offers no particular safety guarantee. The density map can quickly degenerate based on the threshold, while the PURR can still capture the geometry for reasonable ranges of $V_{max}$.}
         \label{fig:purr-viz}
\end{figure}

Finally, we note that the trilinear interpolation of the density (\ref{trilinear}) to compute the integral in (\ref{Eq:CellInt}) introduces a potential source of approximation error.  In practice, this error is much smaller than the over-approximation built into the various voxelization steps, yielding a PURR with a conservative probability of collision.  However, one can remove any doubt about the conservatism of the approach by introducing an upper bound on the trilinear approximation error into the formulation.  We call  this approximation error bound the \textit{collision offset factor}, $\alpha$.   

\begin{definition}[Collision Offset Factor]
    The collision offset factor $\alpha$ is a map-wide upper bound on the difference between (i) the maximum collision probability achieved by integrating the NeRF density $\rho$ over the robot kernel (i.e. the ``true'' collision probability) and (ii) the collision probability using the trilinearly-interpolated density $\hat{\rho}$ from  (\ref{trilinear}),
    \begin{equation*}
    \label{Eq:CollisionOffsetFactor}
\begin{split}
    \alpha \le \min_{i,j,k}\Big\{ & \min_{\mathbf{p} \in v_{ijk}, \mb{R} \in \mathbb{SO}(3)}Pr(X \leq N_{aux}^{max};\int_{B(\mathbf{p}, \mb{R})} \frac{\gamma \rho(\mb{x})}{A_{aux}}\; d\mb{x}) \\&  - \mathbb{P}(v_{ijk})\Big\}.
\end{split}
\end{equation*}
\end{definition}
Finally, if $\alpha$ exists, then (\ref{purr_eqn}) and consequently our PURR are inflated with this collision offset factor to give the PURR a rigorous collision probability guarantee
\begin{equation}
\label{purr_general_eqn}
    \mathbb{P}(v_{ijk}) = Pr(X \leq N_{aux}^{max}; \mathbb{I}_r(v_{ijk})) < \sigma - \alpha.
\end{equation}

\begin{theorem}
Given a collision offset $\alpha$ and the desired collision probability threshold $\sigma$, a robot with position $\mathbf{p}$ in the complement of $\mathbb{P}$ is guaranteed to be probabilistically safe.
\end{theorem}

\begin{proof}
Following the expressions in (\ref{robo_intensity_grid}), the `$\approx$' in the last line becomes `$\ge$' for a collision offset factor, $\alpha$, that satisfies Definition \ref{Eq:CollisionOffsetFactor}.  Therefore, the resulting PURR is an upper bound on the probabilistic collision constraint of (\ref{chance_constraint}).
\end{proof}

\begin{remark}
If the discretizations match between the PURR and a voxel-based NeRF architecture, as in \cite{plenoxels} then $\alpha$ is identically 0. In practice, we still set $\alpha$ to 0 regardless of the NeRF architecture, and we find that the PURR free space is safe invariant to the size of the cells used in trilinear interpolation (Fig. \ref{fig:grid_ablation}). We conclude that $\alpha \approx 0$ (i.e. trilinear interpolation closely approximates the neural network) or $\alpha \geq 0$ (i.e. trilinear interpolation is over-approximating the network). 
\end{remark}

\begin{remark} Although voxel representations are undesirable in memory compared to compact neural networks, we note that the PURR is binary. The memory footprint can be further reduced by using octrees and by storing the PURR with a compression scheme, e.g, with hashing.
\end{remark}
% \subsection{Extensions to State Uncertainty}\label{state_uncertainty}

% \begin{remark}
% The creation of a PURR with state uncertainty $\hat{\mathbb{P}}$ is located in Appendix \ref{app:state_uncertain_purr}. The state uncertain PURR $\bar{\mathbb{P}}$ retains the original probabilistic collision guarantee.
% \end{remark}

\subsection{Path Planning in the PURR}
We now turn to the problem of chance-constrained trajectory optimization through the PURR. In particular, we seek to plan a dynamically feasible path for a robot through the environment such that all points along the trajectory satisfy a chance constraint on collision (rather than enforcing the chance constraints at discrete ``knot points'' along the trajectory). In particular, we seek to find trajectories that are \emph{probabilistically safe}. 

\begin{definition}[Probabilistically Safe Trajectory]
A trajectory is probabilistically safe if all points in the trajectory are point-wise probabilistically safe.
\end{definition}

%% This declares a command \Comment
%% The argument will be surrounded by /* ... */
% \SetKwComment{Comment}{/* }{ */}

% \begin{algorithm}
% \caption{Trajectory generation on a PURR}\label{alg:two}
% \KwData{PURR $P$, initial pos. $\mb{p}_0$, final pos. $\mb{p}_f$}
% \KwResult{Control points $\{\mb{s}_k^i\}$ of prob. safe traj.}
% $\bar{\mb{p}}_\text{discrete} \gets$ \textproc{A$^*^}$(\mb{p}_0, \mb{p}_f)$\;
% $\{BB_i\} \gets$ \textproc{boundingBoxes}$(\bar{\mb{p}}_\text{discrete})$\;
% $\mb{s}_k^i \gets$ \textproc{bezier}$(\mb{p}_0, \mb{p}_f, \{BB_i\})$ \Comment*[r]{QP solve}
% \end{algorithm}

We propose an algorithm containing three components used to create these safe, continuous trajectories. The first step is to find an initial, discrete path through the free space of the PURR by solving a constrained shortest path problem (Fig. \ref{fig:path_explainer}a) from our initial position $\mb{p}_0$ to our final position $\mb{p}_f$ (e.g., using A$^*$). While dynamically infeasible, this path provides a connected, collision-free path through the PURR that we will refine into a smooth, feasible trajectory. The second step is to create a ``tube'' of bounding boxes around this initial path that is not in collision with the PURR (Fig. \ref{fig:path_explainer}b). We opt for this design choice because our emphasis in this paper is in quantifying collision risk in the NeRF in combination with a relatively simple planning scheme. One can introduce more flexible and sophisticated planning schemes \cite{liu-safe-flight-corridors, toumieh-convexdecomposition} to improve on our method. Finally, we  generate a smooth curve connecting our initial and final positions by solving a constrained convex optimization, requiring the curve to lie in the free ``tube'' generated previously (Fig. \ref{fig:path_explainer}c). We call the resulting trajectory planning algorithm Collision Avoidance Through Neural Implicit Probabilisitc Scenes (CATNIPS).  CATNIPS executes in real-time given a pre-computed PURR map.

\subsubsection{A$^*$ Search}
We first find a rough, discrete initial path through the NeRF using an A$^*$ search over the voxel grid defining the PURR free space. Specifically, given an initial position $\mb{p}_0$ and final position $\mb{p}_f$ of the robot, we find a minimum-length (measured in the Manhattan distance) path between the corresponding initial and final voxels. We search a 6-connected graph, i.e., the robot can move into neighboring free voxels of the PURR along the $x$-, $y$-, and $z$-axes. Using A$^*$ with the usual heuristic (distance to goal, not considering collision) yields a connected, collision-free, but dynamically infeasible path from the start to the goal.

\subsubsection{Bounding Box Generation}
%The primary motivation of the bounding boxes is the fact that if the continuous trajectory can live within the union of these boxes that exist in the complement of the PURR, the trajectory will be safe. Additionally, it will be less conservative than a trajectory existing in the sequence of cells found by the initialization if the union of the bounding boxes are a superset of this initialized cell sequence.

We now seek to refine the discrete, collision-free path returned by A$^*$ into a continuous trajectory that is energy-efficient, and dynamically feasible, for our robot. To this end, we first generate a ``tube'' around our A$^*$ path that is both large (so we minimally constrain our trajectory optimization) and lies in the free space of the PURR (so trajectories in this tube still remain collision-free). We represent this tube as a union of bounding boxes, as shown in Fig.~\ref{fig:path_explainer}b. 

To generate these bounding boxes, we first split the A$^*$ trajectory into straight-line segments (i.e., if the path returned by A$^*$ begins by moving along the $z$-axis for 6 voxels, we join these into a single line segment between the start and endpoint of this sequence). We then expand a bounding box around each line segment by ``marching'' each face along its normal direction until it is marginally in collision with the PURR (i.e., at least one cell on the face borders an unsafe cell). To speed collision-checking for the prospective boxes, we convert the PURR to a KD-tree representation for this step.

%With each segment relating to a sequence of voxels, we can grow a bounding box from this initial set of voxels. We grow this box until all the faces are marginally in collision with the PURR, which for the purpose of collision checking is stored as a KD-tree.  

Once this process is complete, we now have one bounding box for each line segment in our original A$^*$ path; the union of these bounding boxes both lies in the free space of the PURR, and contains at least one connected, collision-free path between our initial and final positions. However, for voxel grids with fine spatial resolution, the number of bounding boxes will grow, adding to computational complexity; thus, as a final step we eliminate all ``similar'' bounding boxes (i.e., any bounding box whose volume has sufficient overlap with a bounding box earlier in the sequence) to generate a simplified representation. We find that these simple axis-aligned rectangular bounding boxes are more well-behaved in dense voxel grids where narrow corridors exist, while general polytopic alternatives like \cite{liu-safe-flight-corridors} introduce numerical instability into the planner due to acute corners in the safe corridors (demonstrated by \cite{toumieh-convexdecomposition}).

\subsubsection{Smooth Trajectory Generation via Bézier Curves}
The final step of our planner is to generate a smooth trajectory that lies entirely in the union of the bounding boxes. To do this, we represent our trajectory as a connected series of Bézier curves. In particular, a Bézier curve in $\mathbb{R}^n$, of order $N$, is given by
\begin{equation}
    \mathbf{p}(t; \mathbf{s}_k) = \sum_{k=0}^N \binom{N}{k} (1-t)^{N-k} t^k \mathbf{s}_k,
\end{equation}
where $\mb{s}_k \in \mathbb{R}^n$ are a set of ``control points'' defining the geometry of the curve, and the curve is traced by a free parameter $t \in [0, 1]$. For any $t \in [0, 1]$, the Bézier curve $\mb{p}(t)$ is simply an interpolation of the control points $\mb{s}_k$, which means the parametric curve lies in the convex hull of the control points \cite{joybernstein}. Thus, to generate a probabilistically safe trajectory through the PURR, we find a set of Bézier curves connecting our initial position $\mb{p}_0$ and final position $\mb{p}_f$, whose control points lie in the bounding boxes generated previously; since each Bézier curve must lie in the convex hull of its control points, the entire curve will lie in the complement of the PURR. We note that this is a common method for enforcing safety constraints in the path planning literature \cite{gasparetto2007new}. %\textcolor{red}{Mac: Add a reference for this trick.  Is there a paper that is considered to have introduced this in robotic planning?}

% An immediate result is the linearity of the spline with respect to the control points $\mathbf{s}_k \in \mathbb{R}^3$, which we will utilize to solve for these control points efficiently.

% To enforce safety of these splines, we leverage the \emph{convex hull} property \tc{cite} [], which states that the Bézier curve for $t \in \left[0, 1 \right]$ lives within the convex hull of its control points. If we constrain the control points to lie within a bounding box found by the method above (which themselves are safe), the convex hull is necessarily a subset of the bounding box. As a result, the resulting spline is guaranteed to be a probabilistically safe trajectory.

\begin{figure}[t]
 \centering
 \includegraphics[width=0.49\textwidth]{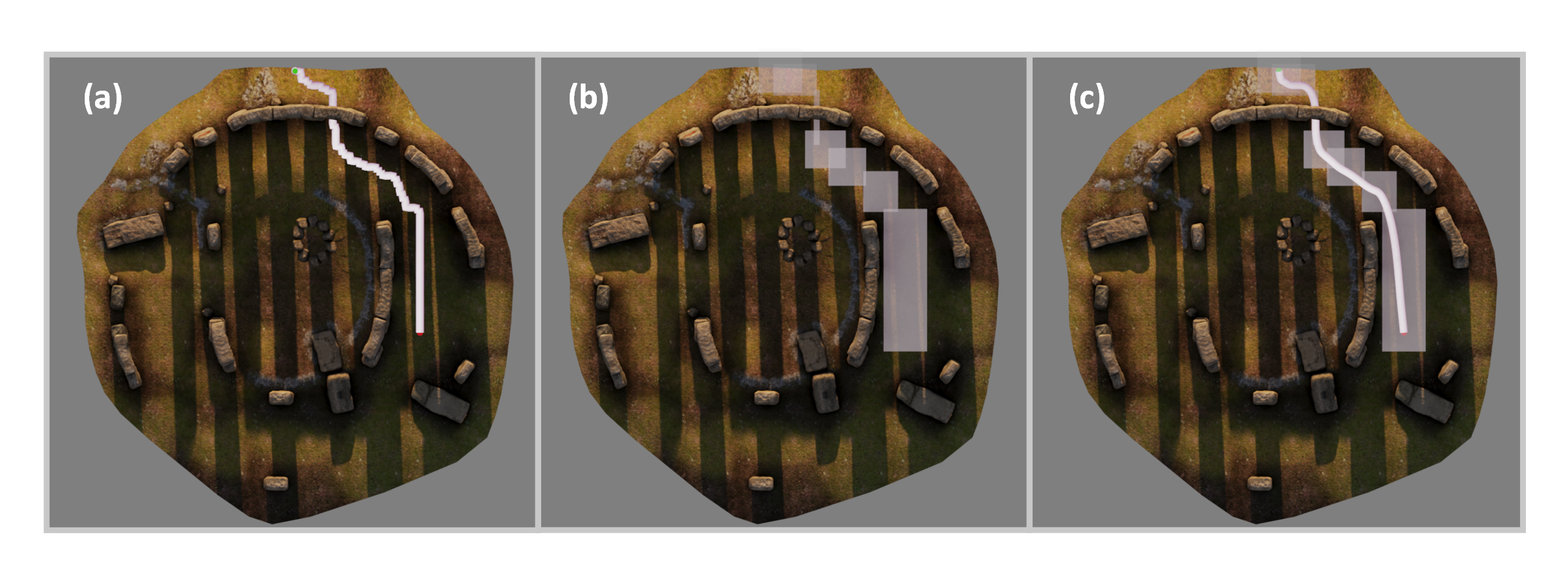}
 \caption{(a) Discrete path returned by A$^*$, (b) Union of bounding boxes containing subsets of the A$^*$ path whilst remaining strictly in the complement of the PURR, (c) Smooth path represented as the union of Bézier curves whose control points lie within a particular bounding box.}
 \label{fig:path_explainer}
\end{figure}

% \subsubsection{Safe Trajectory Solver as a Quadratic Program}
To find this trajectory, suppose we have $L$ bounding boxes $\{\mathbb{B}_1, \ldots, \mathbb{B}_L\}$ generated from the previous step. We then find a set of $L$ Bézier curves with control points given by $\mb{s}_k^i$, constraining all control points of the $i^\text{th}$ curve to lie in the corresponding bounding box $\mathbb{B}_i.$ Since the Bézier curves are linear functions of the control points, we can in turn represent the $i^\text{th}$ curve as $\mathbf{p}_i(t) \equiv \beta(t) \mathbf{s}^i$, where $\mathbf{s}^i \in \mathbb{R}^{n (N+1)}$ is a concatenated vector of all $N + 1$ control points for curve $i$, and $\beta(t): [0, 1] \mapsto \mathbb{R}^{n \times n(N+1)}$ is a coefficient matrix that only depends on the curve parameter $t$. We can similarly represent higher derivatives of the curve as $\mathbf{p}_i^{(d)}(t) = \beta^{(d)}(t)\mathbf{s}^i$. Since we are only concerned with positions and their derivatives, $n = 3$. We refer the reader to \cite{joybernstein} for a more detailed treatment of Bézier curves and splines.

To help smooth the spline and discourage looping behavior, we introduce the objective
\begin{equation*}
    J(\mathbf{s}^1, \ldots \mb{s}^L) = \sum_{i=1}^L \left(\int_0^1 \lvert\lvert \beta^{(d)}(t)\mathbf{s}^i\rvert\rvert_2^2 dt + \sum_{k=0}^{N-1} \lvert\lvert\mathbf{s}^i_k -\mathbf{s}^i_{k+1}\rvert\rvert_2^2
    \right),
\end{equation*}
which is quadratic in our decision variables $\mb{s}^i$. A typical choice is to penalize the snap of the trajectory ($d = 4$) as a proxy for control effort \cite{mellinger_minsnap}. 

To generate our desired trajectory, we then solve the following optimization:
\begin{align}
\min_{\mb{s}^1, \ldots, \mb{s}^L} \quad& J(\mb{s}^1, \ldots,  \mb{s}^L) \label{eq:traj_qp}\\
\text{s.t.} \quad& \mb{s^i_j} \in \mathbb{B}^i, \quad &\forall i \leq L, \; j \leq N \nonumber \\
&\beta^{(d)}(1) \mb{s}^i = \beta^{(d)}(0) \mb{s}^{i+1}, &\forall i \leq L, \; d \leq D \nonumber\\
&\beta(0)\mb{s}^1 = \mb{p}_0, \nonumber\\
&\beta(1) \mb{s}^L = \mb{p}_f \nonumber .
\end{align}
In particular, we constrain the control points of every segment so that $\mb{s}^i$ must lie in the corresponding bounding box $\mathbb{B}_i$, which defines a set of linear inequaities in $\mb{s}^i$. We also enforce continuity of each spline up to a desired derivative $D$, which defines a set of linear equality constraints. Finally, we enforce the boundary conditions of our trajectory, i.e., that the curve begins at our initial position $\mb{p}_0$ and ends at our final position $\mb{p}_f$. We choose to optimize Bézier curves of order $N=8$, to balance the expressiveness of our model (which needs non-trivial derivatives up to order $d=4$) with the number of parameters needed to specify the curve. 

Since our objective is quadratic in the control points, and our constraints are defined by linear inequalities and equality constraints, the resulting optimization \eqref{eq:traj_qp} is a quadratic program (QP) that can be solved in real time.

\begin{corollary}
The trajectory given by the solution of the QP \eqref{eq:traj_qp} is probabilistically safe.
\end{corollary}

\begin{proof}
    The QP \eqref{eq:traj_qp} constrains each Bézier curve to live within a bounding box that is probabilistically safe, rendering each curve safe. The resulting trajectory, given by  the union of the Bézier curves is therefore probabilistically safe.
\end{proof}

\begin{remark}
\label{remark:probabilistically-safe}
    We emphasize that all trajectories are probabilistically safe in that all points in any trajectory satisfies the inter-penetration constraint (\ref{chance_constraint}) with probability $\sigma$. This is not equivalent to the statement that some $\sigma$ fraction of all trajectories do not contain any points that violate the inter-penetration constraint.
\end{remark}

\begin{remark}
    If the robot system dynamics is differentially flat such that its position is a subset of the flat outputs, then the paths generated by the proposed QP \eqref{eq:traj_qp} (with D set to the highest derivative of position in the flat outputs) are dynamically feasible. Therefore a robot tracking a trajectory from this planner remains probabilistically safe. 
\end{remark}

\begin{remark}
    We note that each segment of the trajectory returned by our planner is parameterized by a simple curve parameter $t$, which need not correspond to time. However, since we assume our system is differentiably flat, there exists a time scaling such that the curve is dynamically feasible. To resolve this, we use a simple time rescaling (as in \cite{mellinger_minsnap}) to generate the final trajectory as a function of time.  
    % \textcolor{red}{Let's include a remark on dealing with time (mapping the distance parameter in the Bezier curve to a time duration), referencing Mellinger to note the flexibility here.}
\end{remark}

\section{Numerical Results}
\label{Sec:Results}

In this section, we study our proposed chance-constrained trajectory optimizer on the simulated Stonehenge scene and real Statues and Flightroom environments. The real scenes were captured using a hand-held phone camera with poses extracted from COLMAP. Using our proposed trajectory optimizer, we generate trajectories for a simulated and real quadrotor flying through the scene, and study the safety and conservativeness of trajectories generated across a large number of initial and final conditions. Using the same path planning algorithm, we perform a comparison between the PURR and a baseline voxel occupancy representation obtained by thresholding the raw NeRF density at a desired density level. We also compare these methods to the authors' previous work NeRF-Nav \cite{adamkiewicz2022vision}. Because this method requires an A* initialization, we use the same A* initialization for both the baseline grid and NeRF-Nav. Specifically, the NeRF-Nav A* initialization is generated from the baseline density grid corresponding to the cutoff $\rho = 10^2$ (the most conservative cutoff). 

We demonstrate that both voxel methods are more computationally efficient than NeRF-Nav, and also generate trajectories that are safer (with fewer collisions) and less conservative (shorter paths). Further, we find that our method, CATNIPS, allows the user to set a clearly defined probability threshold for collision. In contrast, the density threshold baseline does not give such a probabilistic guarantee. In other words, similar behavior can be obtained from a density thresholded map, but this requires a user to tune the density threshold through trial and error to reach a desired qualitative level of safety, and thresholds that work for one environment may not generalize to others. Even after tuning to get good empirical behavior, the baseline offers no accompanying safety guarantee.
\subsection{Algorithm Performance}

\begin{figure}[]
     \centering     
     \includegraphics[width=\columnwidth]{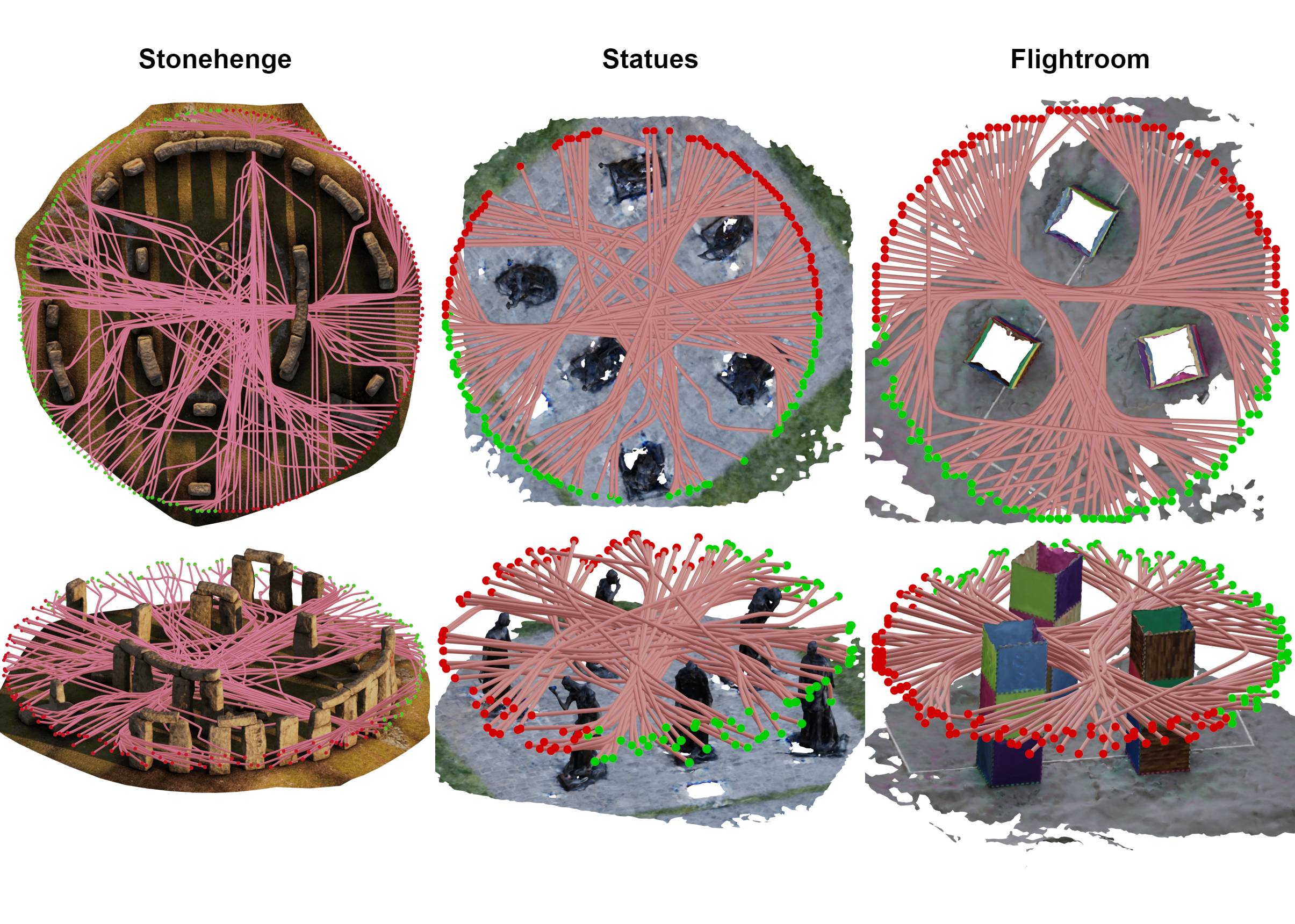}
    \caption{Generated safe paths across 100 configurations from Stonehenge, Statues, and Flightroom, visualized from the top and sides. Both Statues and Flightroom NeRFs were trained from images of the real environments.}
    \label{fig:paths-qualitative}
\end{figure}

\begin{figure*}     
    \centering
     \includegraphics[width=\textwidth]{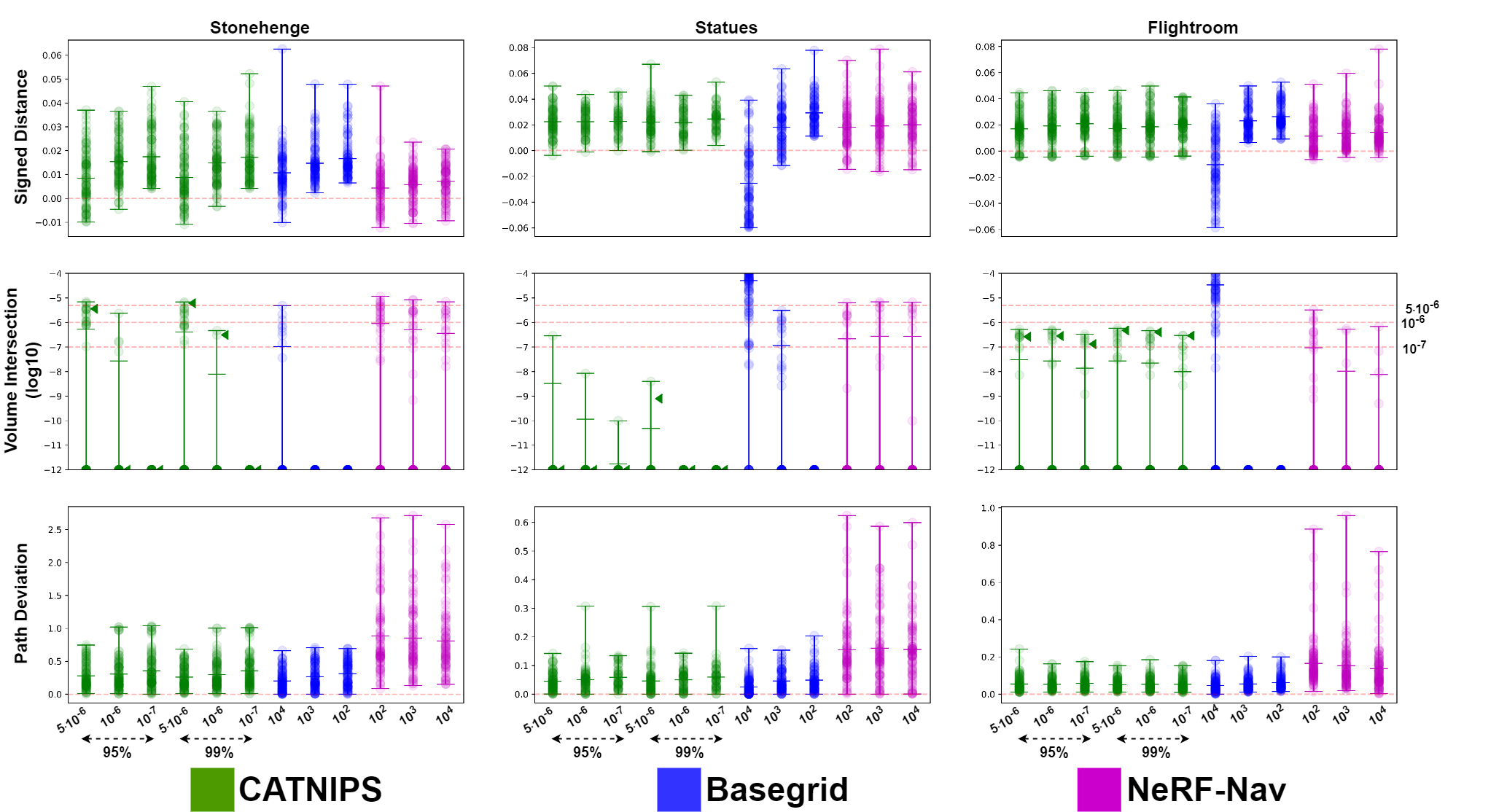}
    \caption{Statistics on distances to obstacles, inter-penetration, and path conservativeness over 100 trajectories for each environment (Stonehenge, Statues, Flightroom). We benchmark our PPP method (three different inter-penetration distances, each at at $95\%$ and $99\%$ probability) against the paths using the baseline grid and NeRF-Nav \cite{adamkiewicz2022vision}. Whiskers indicate max/mean/min over all trajectories, and the density of color represents the spread. Top: The minimum distance to the ground-truth mesh for every trajectory. Our method uses interpretable parameters, such that we can tune for safety (lower volume intersection) without being overly conservative (lower path deviation).  Meanwhile, the effect of the density cutoff in the baseline grid is unpredictable across scenes, and the NeRF-Nav paths can lead to collision violations (low whiskers) and overly-conservative paths. Mid: The maximum inter-penetration volume per trajectory. Although we only make claims about point-wise safety, we see that trajectory-wise safety is approximately satisfied (i.e. arrowheads representing 95, 99$\%$ of trajectories tend to be below the specified inter-penetration indicated by red dotted lines). Bottom: Difference between the minimum length of a straight line path and the executed path. We see that both CATNIPS and the baseline grid are less conservative than NeRF-Nav.}
    \label{fig:comparison}
\end{figure*}

\begin{figure}     
    \centering
     \includegraphics[width=\columnwidth]{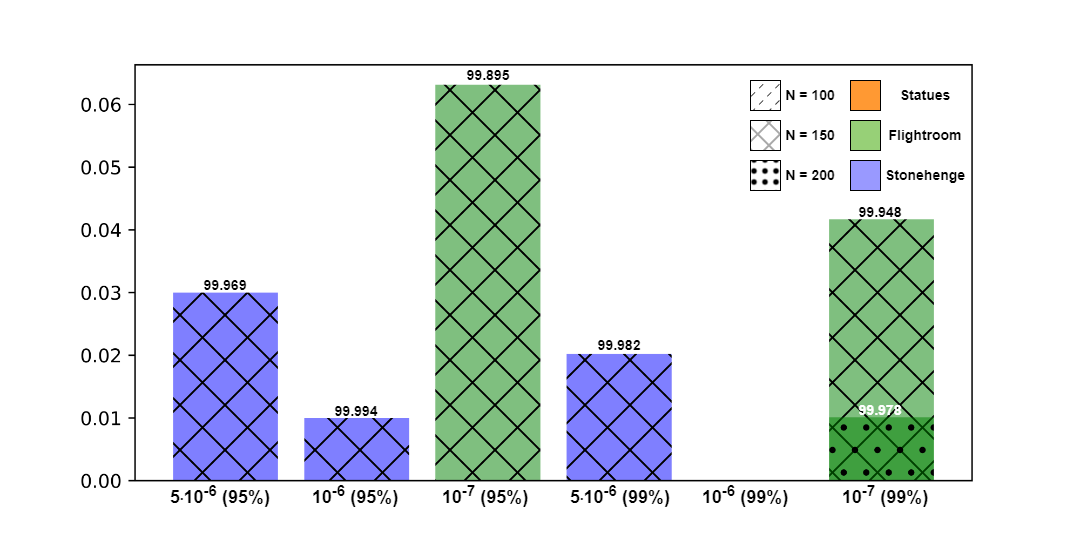}
    \caption{Effect of grid discretizations and environments on safety. Bars represent fraction of trajectories out of 100 that contain at least one point greater than the specified  inter-penetration volume. This ablation is performed over specified inter-penetration volumes, collision probabilities $\sigma = (95\%, 99\%)$, grid discretizations $N = (100, 150, 200)$, and environments (Stonehenge, Statues, Flightroom). The total percent of points across all trajectories below the specified inter-penetration volume for a specific combination of parameters and scenes is indicated above the bar. Combinations not shown means that all points in that setting were below the specified inter-penetration volume. We see that point-wise safety is always satisfied, and trajectory-wise safety is approximately valid.}
    \label{fig:grid_ablation}
\end{figure}

Qualitative results of the proposed method for 100 start and goal locations on a circle (perturbed randomly in the up-down direction) for all 3 environments are shown in Figure \ref{fig:paths-qualitative}. The PURR was generated using a resolution of 150 voxels per side, with probability threshold $\sigma = 0.95$, and $V_{max} = 10^{-6}$ specified volume penetration \footnote{Because the cameras are mapped to be within a unit box, all reported length scales are assumed to be in this normalized system unless explicitly stated.}. Moreover, because our quadrotor system is differentially flat with flat outputs in position, the robot can follow these paths with a standard differential flatness based control pipeline \cite{mellinger_minsnap}.

% All paths are visually collision-free, smooth, and non-degenerate. Figure \ref{fig:purr-viz} shows the resulting PURR of Flightroom at various values of $\sigma$, and the baseline occupancy map computed for a variety of NeRF density thresholds, $\rho$. We find that the PURRs are well-behaved across a wide range of $\sigma$ values, in contrast to the density threshold map. Visually, the density-based map is typically more conservative at low $\rho$ (e.g. no gap exists in the front pillars for $\rho=10^3$, but exists in $\sigma = 10^{-3}$) yet degenerates very quickly with large $\rho$. For example, the floor begins to disappear at $\rho = 10^4$ and is largely gone at $\rho=10^5$ compared to the PURR at $\sigma = 0.25$. Therefore, we argue that the well-behavedness of the PURR contributes to the safety and non-conservatism of the generated paths.

We first compare our method to two baselines: the ``baseline grid'' which computes occupancy from a density threshold, and NeRF-Nav \cite{adamkiewicz2022vision}, a gradient-based trajectory planner for NeRFs in Fig. \ref{fig:comparison}. We analyze their performance on 3 metrics: the minimum signed distance achieved during the trajectory to the ground-truth mesh (negative is within the mesh), the maximum inter-penetration volume achieved during the trajectory, and the difference in lengths between the generated trajectory and the shortest straight line path. The first two metrics quantify safety, while the last quantifies conservativeness. We evaluate the algorithms on the 100 randomized configurations distributed evenly on a circle. 

We choose to display the same 6 parameter combinations for CATNIPS, varying the specific inter-penetration and probability cutoff, over all scenes to demonstrate generalizability and interpretability. The cutoffs we choose to be reasonable thresholds of $95\%$ and $99\%$. The inter-penetration we choose to be some fraction of the robot body. For Stonehenge, $5\cdot 10^{-6}, 10^{-6}, 10^{-7}$ correspond to $13\%, 3\%, .3\%$, while for the real environments, these values correspond to $4\%, .8\%, .08\%$, respectively. 

To benchmark against the baselines, we vary the density cutoffs of the baseline grid and the collision penalty weight in NeRF-Nav. We again stress the lack of interpretable parameters for both NeRF-Nav and the baseline grid paths and their required parameter tuning to get a desired safety performance, which is impossible to know a priori as one does not have access to the ground-truth in reality. However, in order to benchmark the methods in good faith, we choose the density cutoffs on the baseline grid such that the performance on these metrics were similar to those of CATNIPS in a synthetic environment (i.e. Stonehenge), since we have access to its ground-truth mesh. We then use the same cutoffs for the real environments. In our experience, these are also thresholds that are typically used to extract meshes from NeRFs using marching cubes \cite{lorensen1987marchinga}. For NeRF-Nav, the collision penalty weights are chosen so that they are the dominant term in the loss.

We can see that NeRF-Nav trajectories are unsafe when compared to paths generated from either voxel method (negative is in collision with the mesh) over all collision loss weights $(10^2, 10^3, 10^4)$. As we increase the weighting on the collision penalty, we do see that the algorithm can be increasingly safe on average (higher SDF, lower volume intersection). However, such a high collision penalty ($10^4$) will typically cause numerical issues in the trajectory optimizer. Moreover, these trajectories in the worst case are simply less safe than either voxel method. Finally, these trajectories also deviate from the shortest path the farthest, illustrating conservatism and non-smooth paths.

The trajectories derived from the baseline grid can exhibit safe and non-conservative behavior. However, it is clear that the parameters necessary to achieve this behavior cannot be generalized over all scenes. This is evident for the cutoff $\rho = 10^4$, where safety performance in Stonehenge is reasonable, but we see unsafe performance in the real environments. 

We see that our method, CATNIPS, is safe (by construction) and non-conservative. On average, these trajectories are similar in conservativeness compared to the baseline grid paths, while exhibiting reasonable levels of safety and respond as expected to changes in parameters (greater SDF and lower intersection when decreasing specified volume intersection and/or increasing probability threshold). We draw the readers attention to the volume intersection metric (Fig. \ref{fig:comparison}, middle row) for CATNIPS, where trajectory-wise safety is expressed (please see Remark \ref{remark:probabilistically-safe} for the distinction). The arrowhead represents the $\sigma$ quantile over 100 trajectories, while the dotted lines represent the specified volume intersection. While we make no claims on full-trajectory rates of safety (i.e. the arrowhead need not be below the corresponding dotted line), we see that, indeed, a $\sigma$ proportion of the trajectories tends to be completely safe (with no unsafe points existing on the whole trajectory).

We validate the point-wise safety claims we make, as well as ablate CATNIPS over grid discretizations, in Fig. \ref{fig:grid_ablation}. The figure contains runs with parameter combinations of two different collision probabilities, three different specified volume intersections, three different grid discretizations (100, 150, 200), and three different environments. Note that some columns can contain multiple bars, representing different grid discretizations or environments while maintaining the same collision cutoff and volume intersection. Bar heights represent the fraction of all 100 trajectories that contain at least a point with volume intersection higher than what was specified for that combination of parameters. Numbers on top of these bars indicate the percentage of all points in all trajectories that fall below the specified volume intersection for the same parameter setting. Our theoretical claims pertain to the rate across all points on all trajectories (number on top of bars), yet we observe for our method the desired collision rate tends to hold across full trajectories as well (the height of the bars). Combinations not visualized mean there were no points in any trajectories that violated the volume intersection constraint.

Note that the reported percentage of all points being safe (all percentages greater than the probability cutoff) means that our derived point-wise probabilistic safety constraint is validated and that satisfaction of this constraint is invariant to the parameters. This makes (\ref{chance_constraint}), the PURR, and the planning architecture surrounding it generalizable to arbitrary environments and reasonable grid discretizations. This result also implies that the error introduced through trilinear interpolation of neural network-based density fields is small in terms of its impact on safety (i.e. collision offset factor $\alpha \geq 0$). This is especially attractive for real scenes where there is no way to validate safety a priori, and for applications where coarser grid discretizations are necessary for computational performance. Moreover, trajectory-wise safety (like in Fig. \ref{fig:comparison}) is generally satisfied over grid discretizations and scenes.

Here we would like to summarize several subtle points regarding collision violation to the reader. The violation of the specified volume penetration at some points (Fig. \ref{fig:comparison}, \ref{fig:grid_ablation}) is due to both the probabilistic nature of the collision constraint, and due to the fact that the NeRF does not exactly capture the ground truth surface. For the NeRF density to exactly represent the true surface, under our PPP interpretation, it would have to be exactly $0$ outside the surface, and $\infty$ inside.  In this case, collision would be deterministic because the only way to satisfy the chance constraint on collision would be to have no collision.  In practice, a NeRF density field cannot be either $0$ or $\infty$ both because of continuity in the representation and embedded uncertainty about where the true surface is.  Therefore, a collision with the true surface may occur with the prescribed probability. We believe that an accurate representation of uncertainty in perception-based planning must admit for some probability of collision, as there is always
a nonzero probability that perception errors have led to an incorrect estimation of occupancy
in the scene. We further alleviate this “NeRF-to-real” gap by embedding several conservative
approximations into our method through the construction of the PURR. Therefore, in practice,
we collide with the true surface less frequently than required by the collision constraint. Again,
this fact is illustrated in Fig. \ref{fig:grid_ablation}, where the collision constraint satisfaction is very close to 100$\%$ but not conservative enough to be unappealing to use (Fig. \ref{fig:paths-qualitative}, \ref{fig:comparison}). 

\begin{figure}     
    \centering
     \includegraphics[width=\columnwidth]{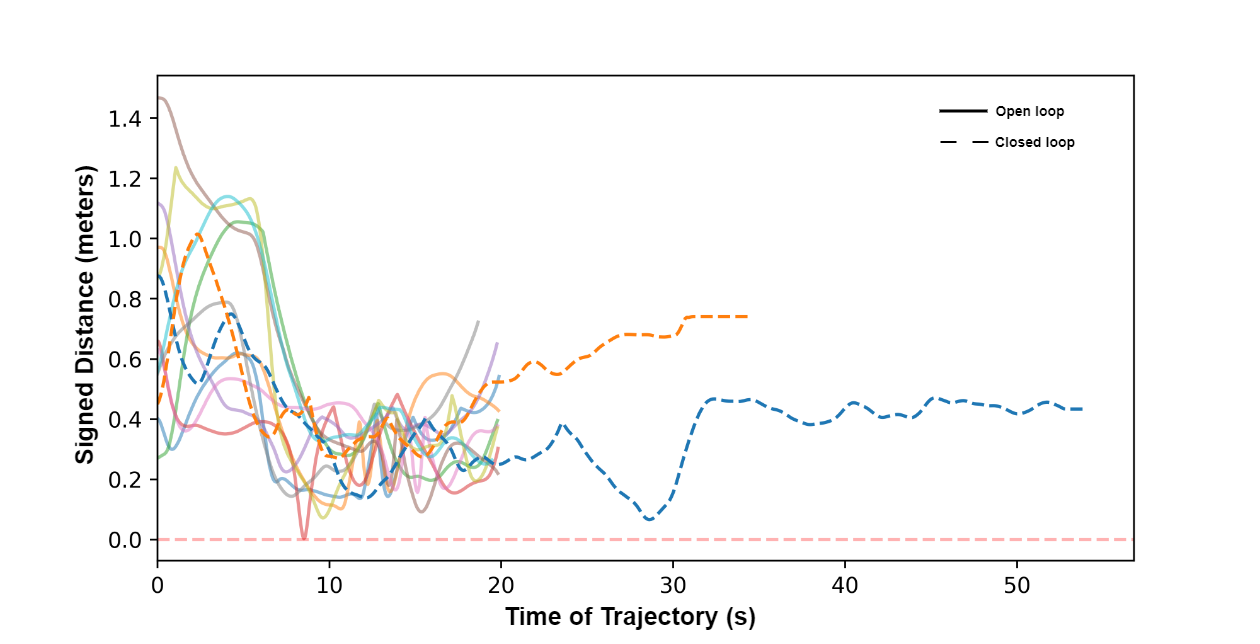}
    \caption{Signed distances along the executed trajectories on a real drone in the Flightroom environment. There were 10 open-loop trajectories, and 2 additional ones with online replanning. All trajectories are above 0 and hence safe.}
    \label{fig:real_experiments}
\end{figure}

Finally, we validate our CATNIPS pipeline on drone hardware experiments in our Flightroom environment. Using a pre-trained static NeRF, we compute 10 trajectories from start-goal points distributed on the perimeter of a circle around the scene, and drive the on-board controller to follow these waypoints (open-loop). Then, we choose two of these start-goal locations and run the CATNIPS A*, bounding box generation, and convex program online (choosing the next predicted point as a waypoint) while simultaneously updating the drone to follow the stream of waypoints (closed-loop). All trajectories are run until they reach the goal location. The open-loop trajectories are about 20 seconds long since they are pre-defined, while the closed loop trajectories can have varying times since they are not pre-defined. We see that all trajectories have a signed distance greater than 0 and are therefore safe (no collision). Even for the open-loop trajectory that comes close to 0, we visually verify a collision-less trajectory. The near-zero SDF is due to over-approximating the drone body with a bounding box when computing the signed distance.  

\subsection{Computation Times}
\begin{table}[!h]
\begin{center}
\begin{threeparttable}
\begin{tabular}{ |p{2.5cm}|p{2cm}|p{1.5cm}|}
\hline
\multicolumn{3}{|c|}{Computation Time (seconds)} \\
\hline
Operation & \textcolor{green}{CATNIPS}/\textcolor{blue}{Basegrid} & \textcolor{magenta}{NeRF-Nav}\\
\hline
Offline &  & \\
\hline
\hspace{3mm} Robot Kernel & 0.002 & 0.002\\
\hline
\hspace{3mm} PURR/Basegrid & 1.11 & 1.11\\
\hline
\hspace{3mm} Gradients & N/A & 9.31${}^{\star}$\\
\hline
Online &  & \\
\hline
\hspace{3mm} A* & 0.16 $\pm$ 0.05 & 0.16\\
\hline
\hspace{3mm} Bounding Box    & 0.12 $\pm$ 0.09 & N/A\\
\hline
\hspace{3mm} B-Spline    & 0.034 $\pm$ 0.029 & N/A\\
\hline
\hspace{3mm} Gradients    & N/A & 0.93${}^{\star \star}$\\
\hline
\end{tabular}
    \begin{tablenotes}
      \small
      \item  ${}^\star$ 1000 gradient steps.
      \item ${}^{\star \star}$ 100 gradient steps.
    \end{tablenotes}
  \end{threeparttable}
\caption{\label{comp. times} Timing results (performed on a laptop with an RTX 4060 GPU) between our voxel methods (CATNIPS and baseline grid) and NeRF-Nav. Because both voxel methods use similar operations, they have identical times. Since NeRF-Nav depends on an A* initialization, it also inherits some computation from the voxel methods. Offline operations only need to be performed once for a static NeRF environment, while online operations need to be performed whenever replanning occurs.}
\end{center}
\end{table}

The implementation of the above algorithms are performed in Pytorch on a laptop with an RTX 4060 GPU. Our method is built on top of NeRFStudio \cite{nerfstudio}. Little effort was made to optimize code for fast computation, so we expect these execution time could be substantially reduced. Moreover, for a fair comparison, we ported NeRF-Nav to NeRFStudio. In general, the planning portion of CATNIPS (A*, bounding box, and Bézier curve generation) operates at around 3 Hz. The operation from querying the NeRF density to the creation of the PURR runs at 1 Hz. The break-down of each operation is shown in Table \ref{comp. times}. As a promising direction for future work, one can further reduce the computation time for creating the PURR and optimizing trajectories in the PURR through parallelization and code optimization on a GPU. \footnote{The A* library \cite{dijkstra3d} we use allows users to pre-process a static voxel grid such that generating the A* initialization is a look-up operation that is near instant in exchange for a one-time processing cost of approximately a second.}. 

Note that the proposed method produces smooth trajectories from the current position to the goal. Also note that in an online re-planning scenario, usually only the next waypoint is tracked before the entire trajectory is updated. Thus, certain parts of our method can be adapted to only consider the vicinity of the robot, trading computation time for suboptimality. In comparison, NeRF-Nav takes longer to converge (if at all) to a reasonable tolerance, without any safety guarantees. We believe this is primarily due to the difficulty of optimizing its highly nonconvex objective and the many queries required to the density neural network within the trajectory optimizer. In order to produce the best performance from NeRF-Nav in terms of safety, we ran the algorithm for 1000 gradient steps.

% Under an online NeRF training scenario, the entire pipeline consisting of PURR creation (0.1 s), A* (0.1 s), and Bézier curve generation (0.02 s) operates at 4 Hz (Table \ref{comp. times}). If we assume a static NeRF scenario, the PURR can be created offline and the A* paths can be cached. More specifically, the A* implementation \cite{dijkstra3d} can simultaneously solve A* from the goal cell to all the other cells. Online replanning then consists of a lookup of the A* path from the current position ($5\times 10^{-4}$ s)  and solving for Bézier curves. Additional optimizations can be done in reducing PURR creation time, as well as investigating methods of creating a trajectory initialization quicker by processing on GPU. 

\section{Conclusion}
\label{Sec:Conclusion}
In this paper, we present a novel method for chance-constrained trajectory optimization through NeRF scenes. We present a method to transform the NeRF into a Poisson Point Process (PPP), which we use to generate rigorous collision probabilities for a robot body moving through the scene. Leveraging this expression for collision probability, we develop a fast method for online trajectory generation through NeRF scenes, which, offline, distills the NeRF density into a voxel-based representation of collision probability called the PURR. Using the PURR, we present an algorithm to plan trajectories represented as Bézier splines that guarantee a robot traversing the spline does not exceed a user-defined maximum collision probability. In numerical experiments, we show our proposed method generates safer and less conservative paths than a state-of-the-art method \cite{adamkiewicz2022vision} for trajectory planning through NeRFs, and also gives more well-behaved and more user interpretable results than a baseline planner that uses a threshold on the NeRF density as a proxy for collision probability. We also demonstrate that our pipeline can run in real-time.

This work opens numerous directions for future research. Since our entire pipeline (both PURR generation and trajectory optimization) can run at real-time rates, our planner could be combined with a NeRF-based state estimator (e.g., \cite{loc-nerf, nice-slam}) to perform active exploration or next-best-view planning on NeRFs, allowing a robot to autonomously explore a novel environment using only onboard vision. Building on navigation, another interesting direction is to tune the collision metric on-line during execution. Because the collision probability is differentiable with respect to the pose, it is possible to tune the collision probability online in response to data collected on-the-fly (e.g., sensed minimum distance to the nearest obstacle).  This could be implemented with auto-differentiation as part of a PyTorch-based planning pipeline. The probabilistic collision framework developed here could have interesting applications to problems like differentiable simulation of rigid bodies represented as NeRFs \cite{cleach2023differentiable} or planning for problems like contact-rich manipulation and locomotion. Finally, since our derivation a PPP from a NeRF is rigorous and generalizable, we hope that our interpretation of NeRFs will be useful to research beyond robotics, for example in computer vision and computer graphics. 

% The CATNIPS pipeline revolves around a novel interpretation of NeRFs as a Poisson Point Process that imbues it with probabilistic properties familiar to roboticists. The pipeline consists of (1) a voxel-based representation of the scene called a PURR that encodes satisfaction of a collision constraint over the robot body and (2) a path-planning algorithm utilizing Bézier curves to generate safe paths through the scene. We validate safety of our pipeline on the Stonehenge scene and show that it outperforms a Marching Cubes proxy and prior work \cite{adamkiewicz2022vision} in safety and conservatism. We also demonstrate that the pipeline can be run in real-time.  

\begin{appendices}

\section{Integration over Trilinearly Interpolated Cells}\label{app:trilinear_analytic}

For a trilinearly interpolated density over a cell $v_{ijk}$ given by (\ref{trilinear}) with local coordinates $(x, y, z) \in $ ($\left[a_x, b_x \right], \left[a_y, b_y \right], \left[a_z, b_z \right]$), the volume integral over the cell can be computed analytically as:

\begin{equation}
\begin{split}
\int_{a_x}^{b_x}\int_{a_y}^{b_y}\int_{a_z}^{b_z} \rho(x) dxdydz  &=  (b_x - a_x)(b_y - a_y)(b_z - a_z) \\
    & + \frac{c_2}{2}(b_y - a_y)(b_z - a_z)(b_x^2 - a_x^2)\\
    & +  \frac{c_3}{2}(b_x - a_x)(b_z - a_z)(b_y^2 - a_y^2)\\
    & + \frac{c_4}{2}(b_x - a_x)(b_y - a_y)(b_z^2 - a_z^2)\\
    & + \frac{c_5}{4}(b_z - a_z)(b_x^2 - a_x^2)(b_y^2 - a_y^2)\\
    & + \frac{c_6}{4}(b_x - a_x)(b_y^2 - a_y^2)(b_z^2 - a_z^2)\\
    & + \frac{c_7}{4}(b_y - a_y)(b_x^2 - a_x^2)(b_z^2 - a_z^2)\\
    & + \frac{c_8}{8}(b_x^2 - a_x^2)(b_y^2 - a_y^2)(b_z^2 - a_z^2)
\end{split}
\end{equation}
where $c_i$ are the coefficients of the trilinear interpolation.

\end{appendices}

% The acknowledgments are automatically included only in the final and preprint versions of the paper.
\section*{Acknowledgements}
We would like to thank Keiko Nagami, Adam Caccavale, Gadi Camps, and Jun En Low for their insights throughout this project.

%===============================================================================
\balance
\bibliographystyle{IEEEtran.bst}
\bibliography{citations.bib}

\end{document}